\documentclass[letterpaper]{article}

\usepackage{times}

\usepackage[accepted]{icml2016}

\usepackage{amsmath}
\usepackage{amsthm}
\usepackage{amssymb}
\usepackage{sectsty}
\usepackage{graphicx}
\usepackage{url}

\usepackage{color, xcolor}
\definecolor{mydarkblue}{rgb}{0,0.08,0.45}
\definecolor{mydarkred}{rgb}{0.54,0,0}
\usepackage[colorlinks=true,
    linkcolor=mydarkblue,
    citecolor=mydarkblue,
    filecolor=mydarkblue,
    urlcolor=mydarkblue,
    pdfview=FitH]{hyperref}

\usepackage{commath}
\usepackage{mdframed}
\usepackage{enumerate}
\usepackage{amsmath}

\DeclareMathOperator*{\argmin}{arg\,min}

\newcommand{\HOAG}{\textsc{Hoag}}

\def\RR{{\mathbb R}}
\def\DD{{\mathcal D}}

\newmdtheoremenv{alg}{Algorithm}
\newmdtheoremenv{theo}{Theorem}
 
\newtheorem{theorem}{Theorem}
\newtheorem{lemma}[theorem]{Lemma}

\newtheorem{innercustomthm}{Theorem}
\newenvironment{customthm}[1]
  {\renewcommand\theinnercustomthm{#1}\innercustomthm}
  {\endinnercustomthm}

\usepackage{natbib}           
\icmltitlerunning{Hyperparameter optimization with approximate gradient}

\begin{document}

\twocolumn[
\icmltitle{Hyperparameter optimization with approximate gradient}
\icmlauthor{Fabian Pedregosa}{f@bianp.net}
\icmladdress{Chaire Havas-Dauphine  ``\'Economie des Nouvelles Donn\'ees'' \\ 
CEREMADE, CNRS UMR 7534, Universit\'e Paris-Dauphine, PSL Research University \\
D\'epartement Informatique de l'\'Ecole Normale Sup\'erieure, Paris}
\icmlkeywords{hyperparameter optimization, machine learning, ICML}
\vskip 0.3in
]

\begin{abstract}
Most models in machine learning contain at least one hyperparameter to control for model complexity. Choosing an appropriate set of hyperparameters is both crucial in terms of model accuracy and computationally challenging. In this work we propose an algorithm for the optimization of continuous hyperparameters using inexact gradient information. An advantage of this method is that hyperparameters can be updated before model parameters have fully converged. We also give sufficient conditions for the global convergence of this method, based on regularity conditions of the involved functions and summability of errors.
Finally, we validate the empirical performance of this method on the estimation of regularization constants of \mbox{$\ell_2$-regularized} logistic regression and kernel Ridge regression. Empirical benchmarks indicate that our approach is highly competitive with respect to state of the art methods.
\end{abstract}

\section{Introduction}

Most models in machine learning feature at least one hyperparameter to control for model complexity. Regularized models, for example, control the trade-off between a data fidelity term and a regularization term through one or several hyperparameters. Among its most well-known instances are the LASSO~\citep{tibshirani1996regression}, in which $\ell_1$ regularization is added to a squared loss to encourage sparsity in the solutions, or $\ell_2$-regularized logistic regression, in which squared $\ell_2$ regularization (known as \emph{weight decay} in the context of neural networks) is added to obtain solutions with small euclidean norm. Another class of hyperparameters are the kernel parameters in support vector machines. For example, the popular radial basis function (RBF) kernel depends on a ``width'' parameter, while polynomial kernels depend on a discrete hyperparameter specifying the degree. Hyperparameters can be broadly categorized into two groups: continuous hyperparameters, such as regularization parameters or the width of an RBF kernel and discrete hyperparameters, such as the degree of a polynomial. In this work we focus on continuous hyperparameters. 


\begin{figure}
\center \includegraphics[width=0.8 \linewidth]{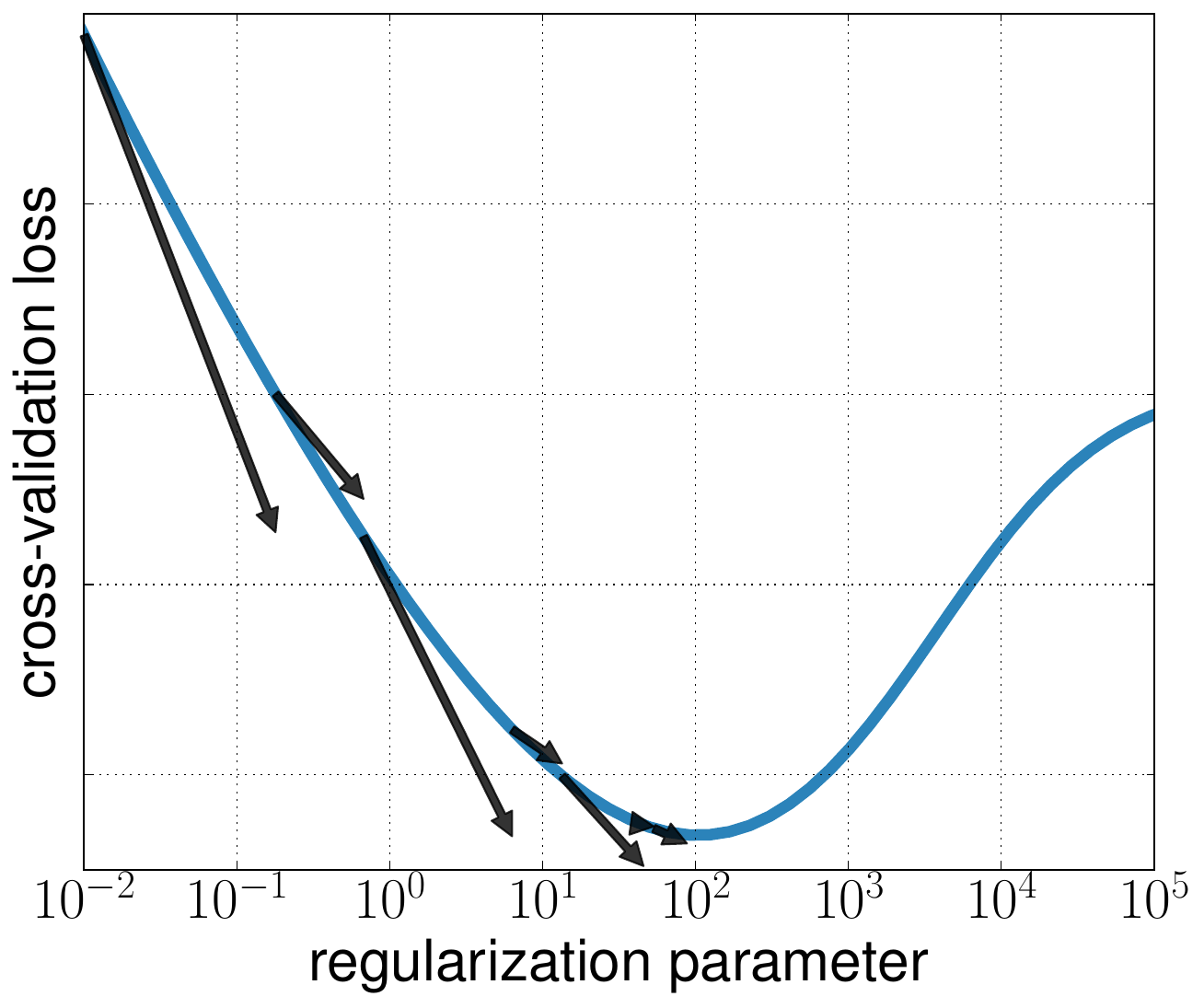}
\caption{Hyperparameter Optimization with approximate gradient. The gradient of the cross-validation loss function with respect to hyperparameters is computed approximately. This noisy gradient is then used to estimate the optimal hyperparameters by gradient descent. A decreasing bound between the true gradient and the approximate gradient ensures that the method converges towards a stationary point.}
\end{figure}

The problem of identifying the optimal set of hyperparameters is known as \emph{hyperparameter optimization}. Hyperparameters cannot be estimated to minimize the same cost function as model parameters, since this would favor models with excessive complexity. For example, if regularization parameters were chosen to minimize the same loss as model parameters, then models with no regularization would always yield the smallest loss. For this reason, hyperparameter optimization algorithms seek to optimize a criterion of model quality which is different from the cost function used to fit model parameters. This criterion can be a goodness of fit on unseen data, such as a \emph{cross-validation} loss, or some criteria of model quality on the train set such as SURE~\citep{stein1981estimation}, AIC/BIC~\citep{liu2011parametric} or Mallows $C_p$~\citep{mallows1973some}, to name a few.

Choosing the appropriate set of hyperparameters has often a dramatic influence in model accuracy and many hyperparameter optimization algorithms have been proposed in the literature. For example, in the widely used \emph{grid-search} algorithm, the model is trained over a range of values for the hyperparameters and the value that gives the best performance on the cross-validation loss is chosen. This not only scales poorly with the number of hyperparameters, but also involves fitting the full model for values of hyperparameters that are very unpromising. Random search~\citep{bergstra2011algorithms} has been proven to yield a faster exploration of the hyperparameter space than grid search, specially in spaces with multiple hyperparameters. However, none of these methods make use of previous evaluations to make an informed decision of the next iterate. As such, convergence to a global minima can be very slow.

In recent years, sequential model-based optimization (SMBO) techniques have emerged as a powerful tool for hyperparameter optimization (see e.g.~\citep{brochu2010tutorial} for an review on current methodologies). These techniques proceed by fitting a probabilistic model to the data and then using this model as an inexpensive proxy in order to determine the most promising location to evaluate next. This probabilistic model typically relies on a Gaussian process regressor but other approaches exist using trees~\citep{bergstra2011algorithms} or ensemble methods~\citep{lacoste2014sequential}. The model is built using only function evaluations, and for this reason SMBO is often considered as a \emph{black-box} optimization method. 


A third family of methods, which includes the method that we present, estimate the optimal hyperparameters using smooth optimization techniques such as gradient descent. 
We will refer to these methods as \emph{gradient-based} hyperparameter optimization methods. These methods use local information about the cost function in order to compute the gradient of the cost function with respect to hyperparameters. However, computing the gradient with respect to hyperparameters has reveled to be a major bottleneck in this approach. For this reason we propose an algorithm that replaces the gradient with an approximation. More precisely, we make the following contributions:
\begin{itemize}
\item We propose a gradient-based hyperparameter optimization algorithm that uses approximate gradient information rather than the true gradient. 
\item We provide sufficient conditions for the convergence of this method to a stationary point.
\item We compare this approach against state-of-the art methods for the task of estimation of regularization and kernel parameter on two different models and three datasets.
\end{itemize}

{\bf Notation} We denote the gradient of a real-valued function by $\nabla$. If this function has several input arguments, we denote $\nabla_i$ its gradient with respect to the $i$-th argument. Similarly, $\nabla^2$ denotes the Hessian and $\nabla^2_{i, j}$ denotes the second order differential with respect to variables $i$ and $j$. For functions that are not real-valued, we denote its differential by $\Dif$. We denote the projection operator onto a set $\DD$ by $P_\DD$. That is,
$
P_\DD(\alpha) \triangleq \argmin_{\lambda \in \DD}{\|\alpha - \lambda\|^2}
$, where $\|\cdot\|$ denotes the euclidean norm for vectors. 

Throughout the paper we take the convention of denoting real-valued functions with lowercase letters (such as $f$ and $g$) and vector-valued functions with uppercase letters (such as $X$). Model parameters are denoted using lowercase Latin letters (such as $x$) while hyperparameters are denoted using Greek lowercase letters (such as $\lambda$).

\subsection{Problem setting}

As mentioned in the introduction, the goal of hyperparameter optimization is to choose the hyperparameters $\lambda$ that optimizes some criteria, such as a cross-validation loss or a SURE/AIC/BIC criteria. We will denote this criteria by $f: \RR^s \to \RR$, where $s$ is the number of hyperparameters. In its simplest form, the hyperparameter optimization problem can be seen as the problem of minimizing the cost function $f$ over a domain $\mathcal{D} \subseteq \RR^s$. Some approaches, such as sequential model-based optimization, only require function evaluations of this cost function. The methods we are interested in however use local information of the objective function.



The cost function $f$ (e.g. the cross-validation error) depends on the model parameters, which we will denote by $X(\lambda)$. These are commonly not available in closed form but rather defined implicitly as the minimizers of some cost function that we will denote ${h(\cdot, \lambda): \RR^p \to \RR}$, where $p$ is the number of model parameters. This makes the hyperparmater optimization problem can be naturally expressed as a nested or \emph{bi-level optimization} problem:
\begin{equation}\label{eq:general_ho}\tag{{HO}}
\begin{split}
\argmin_{\lambda \in \DD} \left\{f(\lambda) \triangleq g(X(\lambda), \lambda) \right\} \\
\text{s.t. } X(\lambda) \in \argmin_{x \in \RR^p} h(x, \lambda) \quad,
\end{split}
\end{equation}
where the minimization over $h$ is commonly referred to as the \emph{inner optimization} problem. A notable example of hyperparameter optimization problem is that of regularization parameter selection by cross-validation. For simplicity, we restrict the discussion to the case of simple or hold-out cross-validation, where the dataset is split only once, although the methods presented here extend naturally to other cross-validation schemes. In this setting, the dataset is split in two: a train set (denoted $\mathcal{S}_{\text{train}}$) and a test or hold-out set (denoted $\mathcal{S}_{\text{test}})$. In this case, the outer cost function is a goodness of fit or loss on the test set, while the inner one is a trade-off between a data fitting term on the train set and a penalty term. If the penalty term is a squared \mbox{$\ell_2$-norm}, then the problem adopts the form:
\begin{equation}\label{eq:ml_ho_problem}
\begin{aligned}
\argmin_{\lambda \in \DD} & ~\text{loss}(\mathcal{S}_{\text{test}}, X(\lambda)) \\
\text{s.t. } X(\lambda) \in \argmin_{x \in \RR^p} & ~\text{loss}(\mathcal{S}_{\text{train}}, x) + e^{\lambda} \|x\|^2 \quad.
\end{aligned}
\end{equation}
The trade-off in the inner optimization between the goodness of fit term and the penalty term is controlled through the hyperparamter $\lambda$. Higher values of $\lambda$ bias the model parameters towards vectors with small euclidean norm, and the goal of the hyperparameter optimization problem is to find the right trade-off between these two terms. The parametrization of the regularization parameter by an exponential ($e^{\lambda}$) in Eq.~\eqref{eq:ml_ho_problem} might seem unusual, but given that this regularization parameter is commonly optimized over a log-spaced grid, we will find this parametrization useful in later sections.




Turning back to the general problem~\eqref{eq:general_ho}, we will now describe an approach to compute the derivative of the cost function $f$ with respect to hyperparameters.
This approach, which we will refer to as implicit differentiation~\citep{larsen1996design,bengio2000gradient,NIPS2007_3286}, relies on the observation that under some regularity conditions it is possible to replace the inner optimization problem by an implicit equation. For example, if $h$ is smooth and verifies that all stationary points are global minima (as is the case for convex functions), then the values $X(\lambda)$ are characterized by the implicit equation $\nabla_1 h(X(\lambda), \lambda) = 0$. 
Deriving the implicit equation with respect to $\lambda$ leads to the equation $\nabla^2_{1, 2} h + \nabla^2_{1} h \cdot \Dif X = 0$, which, assuming $\nabla^2_{1} h$ invertible, characterizes the derivative of $X$. The chain rule, together with this equation, allows us to write the following formula for the gradient of $f$:
\begin{equation}\label{eq:grad_f_full}
\begin{aligned}
\nabla f &= \nabla_2 g +  (\Dif X)^T \nabla_1 g\\
&= \nabla_2 g - \left(\nabla^2_{1, 2} h\right)^T \left(\nabla^2_{1} h\right)^{-1} \nabla_1 g \quad.
\end{aligned}
\end{equation}
This formula allows to compute the gradient of $f$ given the following quantities: model parameters $X(\lambda)$ ($g$ and $h$ are evaluated at $(X(\lambda), \lambda)$) and $\left(\nabla^2_{1} h\right)^{-1} \nabla_1 g$, which is usually computed as the solution to the linear system $\left(\nabla^2_{1} h\right)z = \nabla_1 g$ for $z$. In the section that follows, we present an algorithm that relaxes the condition of both knowledge of the exact model parameters and exact solution of the linear system.

\section{\HOAG: Hyperparameter optimization with approximate gradient}\label{scs:HOAG}

As we have seen in the previous section, computing an exact gradient of $f$ can be computationally demanding. In this section we present an algorithm that uses an approximation, rather than the true gradient, in order to estimate the optimal hyperparameters. This approach yields a trade-off between speed and accuracy: a loose approximation can be computed faster but might result in slow convergence or even divergence of the algorithm. At iteration $k$, this trade-off is balanced by the tolerance parameter $\varepsilon_k$. The sequence of tolerance parameters $\{\varepsilon_1, \varepsilon_2, \ldots \}$ will turn out to play a major role in the convergence of the algorithm, 
although the time being, we will treat it as free parameter.
We now describe our main contribution, the \HOAG\ algorithm:
\vbox{  
\begin{alg}[\HOAG]\label{alg:hoag}
At iteration $k=1, 2,\ldots$ perform the following:
\begin{enumerate}[(i)]
\item Solve the inner optimization problem up to tolerance $\varepsilon_k$. That is, find $x_k$ such that $$\norm{X(\lambda_k) - x_k} \leq \varepsilon_k\quad.$$
\item Solve the linear system $\nabla_{1}^2 h(x_k, \lambda_k) q_k = \nabla_1 g(x_k, \lambda_k)$ for $q_k$ up to tolerance $\varepsilon_k$. That is, find $q_k$ such that
$$
\norm{\nabla_{1}^2 h(x_k, \lambda_k) q_k - \nabla_1 g(x_k, \lambda_k)} \leq \varepsilon_k \quad.
$$
\item Compute approximate gradient $p_k$ as
$$
p_k = \nabla_2 g(x_k, \lambda_k) - \nabla^2_{1, 2} h(x_k, \lambda_k)^T q_k \quad,
$$
\item Update hyperparameters:
$$
\lambda_{k+1} = P_{\DD}\left(\lambda_k - \frac{1}{L} p_k\right) \quad.
$$
\end{enumerate}
\end{alg}
}

This algorithm consists of four steps. The first two steps of the algorithm compute approximations to the quantities used in Eq.~\eqref{eq:grad_f_full} to compute the gradient of $f$. However, since these are not computed to full accuracy, $p_k$, computed in step $(iii)$ is a noisy estimate of the gradient. This approximation is then used as a replacement of the true gradient in a projected gradient-descent ($iv$) iteration.

This procedure requires access to three quantities at iteration $k$: a $\varepsilon_k$-optimal solution to the inner optimization problem which can be computed with any solver, the first-order derivatives of $g$, ($\nabla_1 g, \nabla_2 g$), and an $\varepsilon_k$-optimal solution to a linear system involving $\nabla^2_{1} h$. In practice, this system is solved using a conjugate-gradient method, which only requires access to the matrix $\nabla^2_{1} h$ through matrix-vector products. For example, in machine learning problems such as the ones introduced in Eq.~\eqref{eq:ml_ho_problem}, the quantity $\nabla_1^2 h$ corresponds to the Hessian of the inner optimization problem. Efficient schemes for multiplication by the Hessian can be derived for least squares, logistic regression~\citep{lin2008trust} and other general loss functions~\citep{pearlmutter1994fast}.

\subsection{Related work}\label{scs:related_work}


There exists a large variety of hyperparameter optimization methods, and a full review of this literature would be outside the scope of this work. Below, we comment on the relationship between \HOAG\ and some of the most closely related methods. 

Regarding gradient-based hyperparameter optimization methods we will distinguish two main approaches, implicit differentiation and iterative differentiation, depending on how the gradient with respect to hyperparameters is computed.

{\bf Implicit differentiation}. This approach consists in deriving an implicit equation for the gradient using the optimality conditions of the inner optimization problem (as we did in Eq.~\eqref{eq:grad_f_full}). Originally motivated by the problem of setting the regularization parameter in the context of neural networks~\citep{larsen1996design,larsen1998adaptive,bengio2000gradient}, has also been applied to the problem of selecting kernel parameters~\citep{chapelle2002choosing,seeger2008cross} or multiple regularization parameters in log-linear models~\citep{NIPS2007_3286}. This approach has also been successfully applied to the problem of image reconstruction~\citep{kunisch2013bilevel,calatroni2015bilevel}, in which case the simplicity of the cost function function allows for a particularly simple expression of the gradient with respect to hyperparameters.

{\bf Iterative differentiation}. In this approach, the gradient with respect to hyperparameters is computed by differentiating each step of the inner optimization algorithm and then using the chain rule to aggregate the results. Since the gradient is computed after a finite number of steps of the inner optimization routine, the estimated gradient is naturally an approximation to the true gradient. This method was first proposed by \citet{domke2012generic} and later extended to the setting of stochastic gradient descent by~\citet{MacDuvAda2015hyper}. We note also that contrary to the implicit differentiation approach, this method can be applied to problems with \mbox{non-smooth} cost functions~\citep{deledalle2014stein,ochs2015bilevel}.

\HOAG, while belonging to the class of implicit differentiation methods, is related to iterative differentiation methods in that it allows the gradient with respect to hyperparameters to be computed approximately. 




Finally, we note that similar approaches have also been considered in the setting of {\bf sequential model-based optimization}. \citet{swersky2014freeze} proposes an approach in which the inner optimization is ``freezed'' whenever the method decides that the current hyperparameter values are not promising. It does so by introducing a prior on training curves as a function of input hyperparameters. This approach however requires to make strong assumptions on the shape of the training curves which gradient-based methods do not make.




\section{Analysis}\label{scs:analysis}

In this section we will prove that the summability of the tolerance sequence $\{\varepsilon_i\}_{i=1}^{\infty}$ is sufficient to guarantee convergence of the iterates in \HOAG. The analysis of this algorithm is inspired by the work of~\citet{d2008smooth,schmidt2011convergence, friedlander2012hybrid} on inexact-gradient algorithms for convex optimization.

We will start this section by enumerating the regularity conditions that we assume for the hyperparameter optimization problem. The following conditions are assumed through the section:
\begin{itemize}
\item {(A1) $L$-smoothness}. We assume that the first derivatives of $g$ and the second derivatives of $h$ are Lipschitz continuous functions.
\item {(A2) Nonsingular Hessian}. We assume that the matrix $\nabla_1^2 h$, which corresponds to the Hessian of the inner optimization problem, is invertible at the values $(X(\lambda), \lambda), \lambda \in \mathcal{D}$.
\item {(A3) Convex compact domain}. The domain under which the hyperparameters are optimized, $\DD$, is a convex non-empty and compact subset of $\RR^s$.

\end{itemize}

These assumptions are verified by many models of interest. For example, for the problem of estimation of regularization parameters of Eq.~\eqref{eq:ml_ho_problem}, it allows twice-differentiable loss functions such as logistic regression or least squares (assumption A1) and strongly convex penalties (A2), such as squared $\ell_2$ regularization. Note that condition (A2) need not be verified on all its domain, only on the points $(X(\lambda), \lambda)$, which would allow in principle to consider models that are defined through a non-convex cost functions. Assumption (A3) requires that the domain of the hyperparameters is a convex compact domain. In practice, hyperparameters are optimized over a $s$-dimensional interval, i.e., a domain of the form ${\DD = [a_1, b_1]\times \cdots [a_s, b_s]}$. Our analysis however only require this domain to be convex and compact, a constraint that subsumes $s$-dimensional intervals.

The rest of the section is devoted to prove (under conditions) the convergence of \HOAG.
The proof is divided in two parts. First, we will prove that the difference between the true gradient and the approximate gradient is bounded by $\mathcal{O}(\varepsilon)$ (Theorem~\ref{thm:bound_gradient}) and in a second part we will prove that if the sequence $\{\varepsilon_i\}_{i=1}^\infty$ is summable, then this implies the convergence to a stationary point of $f$ (Theorem~\ref{thm:convergence}). Because of space limitation, the proofs are omitted and can be found in Appendix~\ref{scs:appendix_analysis}.

\vbox{
\begin{theorem}[The gradient error is bounded]\label{thm:bound_gradient}
For sufficiently large $k$, the error in the gradient is bounded by a constant factor of $\varepsilon_k$. That is,
$$
\norm{\nabla f(\lambda_k) - p_k} = \mathcal{O}(\varepsilon_k) \quad.
$$
\end{theorem}
}

This theorem gives a bound on the gradient from the sequence that bounds the inner optimization and the linear system solution. It will be the key ingredient in order to show convergence to a stationary point, which is the main result of this section. As convergence criterion, we consider the norm of the (scaled) gradient mapping $\|\lambda_k - P_{\DD}\big(\lambda_k - \frac{1}{L} \nabla f(x_k)\big)\|$, a generalization of the gradient that ensures convergence to a stationary point \citep[\S 1.4.4]{beck2009gradient}.

\begin{theorem}\label{thm:convergence} If the tolerance sequence is summable, that is, if $\{\varepsilon\}_{i=1}^n$ is positive and verifies
$$
\sum_{i=1}^\infty \varepsilon_i < \infty \quad,
$$
then the limit of the \HOAG\ iterates verifies the stationary point condition:
$$
\lim_{k \to \infty} \|\lambda_k - P_{\DD}\big(\lambda_k - \frac{1}{L} \nabla f(x_k)\big)\| = 0\,.
$$
Furthermore, if the iterates belong to the domain for sufficiently large $k$, then
$$
\lim_{k \to \infty} \|\nabla f(x_k)\| = 0\,.
$$
\end{theorem}

{\color{mydarkred}{\bf Update 2022.11} The statement above has been edited with respect to the original (published) manuscript.\footnote{\url{http://proceedings.mlr.press/v48/pedregosa16.pdf}} The original proof wrongly claimed the convergence of the iterates $\lambda_k$ to its limit. The revised statement removed this claim and only shows the convergence of the gradient mapping.}

This results gives sufficient conditions for the convergence of \HOAG. The summability of the tolerance sequence suggest several natural candidates for this sequence, such as the quadratic sequence, $\varepsilon_k = k^{-2}$ or the exponential sequence, $\varepsilon_k = \rho^k$, with $0 < \rho < 1$. We will empirically evaluate different tolerance sequences on different problems and different datasets in the next section.

\section*{Experiments}\label{scs:experiments}

In this section we compare the empirical performance of \HOAG. We start by discussing some implementation details such as the choice of step size. Then, we compare the convergence of different tolerance decrease strategies that were suggested by the theoretical analysis. In a third part, we compare the performance of \HOAG\ against other hyperparameter optimization methods.

{\bf Adaptive step size}. Our algorithm relies on the knowledge of the Lipschitz constant $L$ for the cost function $f$. However, in practice this is not known in advance. Furthermore, since the cost function is costly to evaluate, it is not feasible to perform backtracking line search. To overcome this we use a procedure in which the step size is corrected depending on the gain estimated from the previous step. In the experiments we use this technique although we do not have a formal analysis of the algorithm for this choice of step size.

Let $\Delta_k$ denote the distance between the current iterate and the past iterate, $\Delta_k = \|\lambda_k - \lambda_{k-1}\|$. The $L$-smooth property of the function $g$, together with Lemma~\ref{thm:bound_gradient}, implies that there exists a constant $M > 0$ such that the following inequality is verified:
\begin{equation} \label{eq:condition_line_search}
\begin{aligned}
g(\lambda_k, x_k) \leq ~&g(\lambda_{k-1}, x_{k-1}) + C \varepsilon_k + \\
&\varepsilon_{k-1} (C + M) \Delta_k - L\Delta_k^2 \quad,
\end{aligned}
\end{equation}
where $C$ is the Lipschitz constant of $g$ (for loss functions such as logistic or least squares this can easily be computed from the data). This inequality can be derived from the properties of \mbox{$L$-smooth} functions, and the details can be found in Appendix~\ref{scs:appendix_experiments}. The procedure consists in decreasing the step (multiplication by $\alpha < 1$) whenever the equation is not satisfied and to increase it (multiplication by $\beta > 1$) whenever the equation is satisfied to ensure that we are using a step size as large as possible. The constants that we used in the experiments are $M=1, \alpha = 0.5, \beta = 1.05$.


{\bf Stopping criterion}. The stopping criterion given in Algorithm~\ref{alg:hoag} depends on $X(\lambda)$ which is generally unknown. However, for objective functions in the inner optimization which are $\mu$-strongly convex ($\mu / 2$ can be taken as the amount of regularization in $\ell_2$-regularized objectives), it is possible to lower bound the quantity ${\|X(\lambda_k) - x_k\|}$ by $\mu^{-1} \|g'(\lambda_k, x_k)\|$. Hence, it is sufficient to ensure ${ \mu^{-1}\|g'(\lambda_k, x_k)\|  \leq \varepsilon}$. Details can be found in Appendix~\ref{scs:appendix_experiments}. 

{\bf Initialization}. The previous sections tells us how to adjust the step size but relies on an initial value of this parameter. We have found that a reasonable initialization is to initalize it to $L = \norm{p_1}$ so that the first update in \HOAG\ is of magnitude at most $1$ (it can be smaller due to the projection), where $p_1$ is the approximate gradient on the first step. The initialization of the tolerance decrease sequence is set to $\varepsilon_1 = 0.1$. We also limit the maximum precision to avoid numerical instabilities to $10^{-12}$, which is also the precision for ``exact'' methods, i.e., those that do not use a tolerance sequence. The initialization of regularization parameters is set to $0$ and the width of an RBF kernel is initialized to $- \log(\text{n\_feat})$, where n\_feat is the number of features or dimensionality of the dataset. 

Although \HOAG\ can be applied more generally, in our experiments we focus on two problems: $\ell_2$-regularized logistic regression and kernel Ridge regression. We follow the setting described in Eq.~\eqref{eq:ml_ho_problem}, in which an initial dataset is partitioned into two sets, a train set $\mathcal{S}_{\text{train}} = \{(b_i, a_i)\}_{i=1}^n$ and a test set $\mathcal{S}_{\text{test}} = \{(b'_i, a'_i)\}_{i=1}^m$, where $a_i$ denotes the input features and $b_i$ the target variables.

The first problem consists in estimating the regularization parameter in the widely-used $\ell_2$-regularized logistic regression model. In this case, the loss function of the inner optimization problem is the regularized logistic loss function. In the setting of classification, the validation loss or outer cost function is commonly the zero-one loss. However, this loss is non-smooth and so does not verify assumption (A1). To overcome this and following~\citep{NIPS2007_3286}, we use the logistic loss as the validation loss. This yield a problem of the form:
\begin{equation}\label{eq:logistic_loss}
\begin{aligned} \argmin_{\lambda \in \DD} &\sum_{i=1}^m \psi({b'}_i {a'}_i^T X(\lambda)) \\
\text{s.t.}~X(\lambda) \in \argmin_{x \in \RR^p} &\sum_{i=1}^n \psi(b_i a_i^T x) + e^{\lambda} \|x\|^2 \quad,
\end{aligned}
\end{equation}
where $\psi$ is the logistic loss, i.e., $\psi(t) = \log(1 + e^{-t})$.
The second problem that we consider is that of kernel Ridge regression with an RBF kernel. In this setting, the problem contains two hyperparameters: the first hyperparameter ($\lambda_1$) controls the width of the RBK kernel and the second hyperparameter ($\lambda_2$) controls the amount of regularization. The inner optimization depends on the kernel through the kernel matrix, formed by computing the kernel of all pairwise input samples. We denote such matrix as $K(\gamma)_{\text{train}}$, where the $(i,j)$ entry is given by $k(a_i, a_j, \gamma)$, where $k$ is the RBF kernel function: $k(a_i, a_j, \gamma) = \exp(- \gamma \| a_i - a_j\|)$. Similarly, the outer optimization also depends on the kernel through the matrix $K(\gamma)_{\text{test}}$, where its entries are the kernel product between features from the train set and features from the test set, that is, $k(a_i, a'_j, \gamma)$. Denoting the full hyperparameter vector as $\lambda = [\lambda_1, \lambda_2]$, the kernel matrix on the train set as, the full hyperparameter optimization problem takes the form
\begin{equation}\label{eq:kernel_ridge}
\begin{aligned} &\argmin_{\lambda \in \DD} ~\norm{b - K_{\text{test}}(e^{\lambda_1}) X(\lambda)}^2 \\
\text{s.t.}~ &\left(K_{\text{train}}(e^{\lambda_1}) + e^{\lambda_2} I \right) X(\lambda) = b \quad,
\end{aligned}
\end{equation}
where for simplicity the inner optimization is already set as an implicit equation. Note that in this setting, and unlike in the logistic regression problem, the outer optimization function depends on the hyperparameters not only through the model parameters $X(\lambda)$ but also through the kernel matrix.

The solver used for the inner optimization problem of the logistic regression problem is L-BFGS~\citep{liu1989limited}, while for Ridge regression we used a linear conjugate descent method. In all cases, the domain for hyperparameters is the $s$-dimensional interval $[-12, 12]^s$.

\begin{figure*}
\includegraphics[width=.33\linewidth]{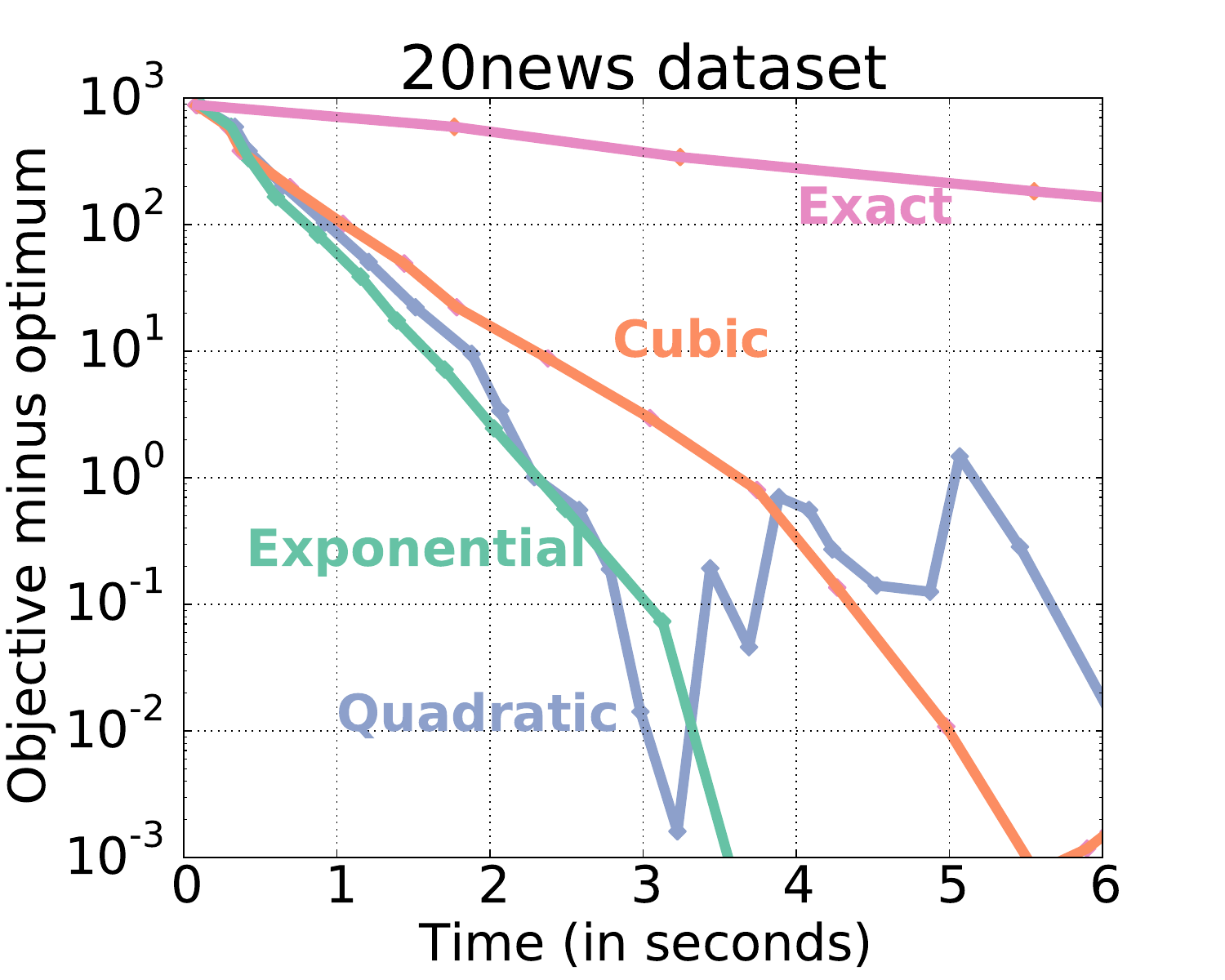}
\includegraphics[width=.33\linewidth]{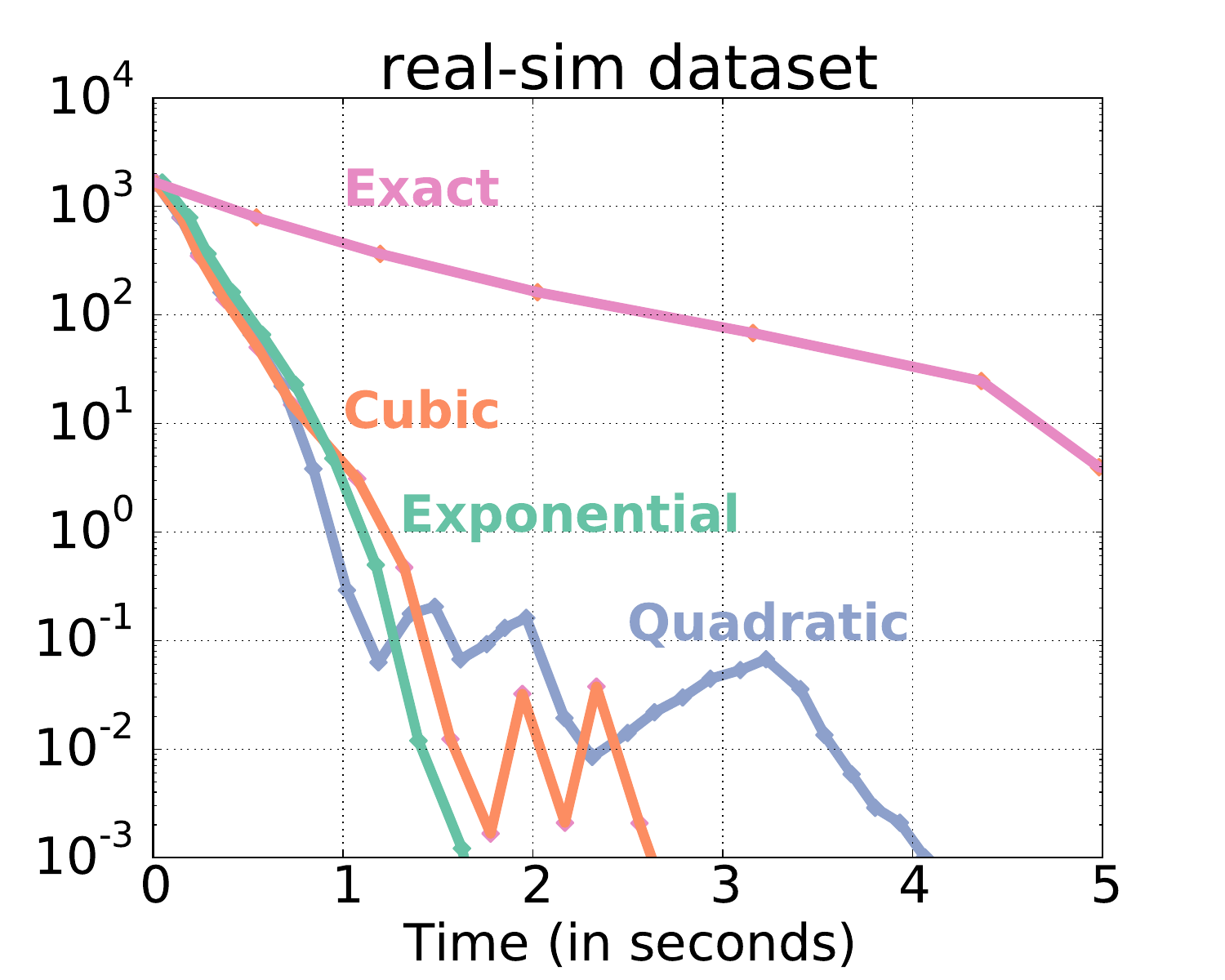}
\includegraphics[width=.33\linewidth]{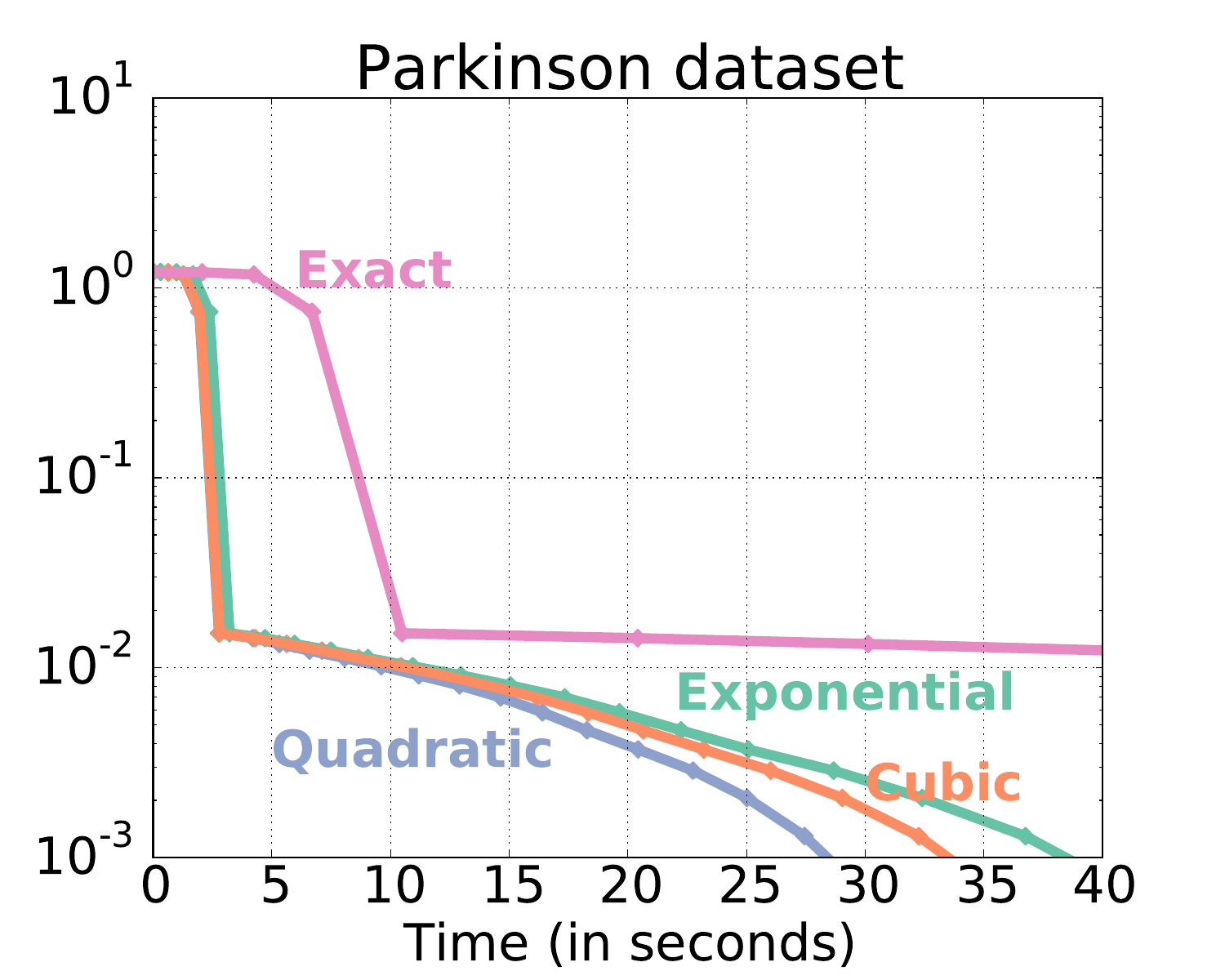}
\caption{{\bf Tolerance decrease strategies}. Suboptimality as a function of time for different tolerance decrease strategies. The decrease sequences considered are quadratic ($0.1 k^{-2}$), cubic ($0.1 k^{-3}$), exponential ($0.1 \times 0.9^{k}$) and exact (gradient is computed to full 
accuracy at every iteration). Non-exact methods exhibit smaller cost per iteration, which results in faster convergence.}\label{fig:bench_tolerance}
\end{figure*}

For the experiments, we use four {\bf different datasets}. The dataset 20news and real-sim are studied with an $\ell_2$-regularized logistic regression model (1 hyperparameter) while the Parkinson dataset using a Kernel ridge regression model (2 hyperparameters). The MNIST dataset is investigated in a high-dimensional hyperparameter space using a similar setting to~\citep[\S 3.2]{MacDuvAda2015hyper} and reported in in Appendix~\ref{scs:appendix_experiments}. Datasets and models are described in more detail in Appendix~\ref{scs:appendix_experiments}.


In all cases, the dataset is randomly split in three equally sized parts: a train set, test set and a third validation set that we will use to measure the generalization performance of the different approaches.

\subsection{Tolerance decrease sequence}

We report in Figure~\ref{fig:bench_tolerance} the convergence of different tolerance decrease strategies. From Theorem~\ref{thm:convergence}, the sole condition on these sequences is that they are summable. Three notable examples of summable sequences are the quadratic, cubic and exponential sequences. Hence, we choose one representative of each of these strategies. More precisely, the decrease sequences that we choose are a quadratic decrease sequence of the form $\varepsilon_k = 0.1 \times k^{-2}$, a cubic one of the form $\varepsilon_k = 0.1 \times k^{-3}$ and an exponential of the form $\varepsilon_k = 0.1 \times (0.9^{k})$. The value taken as true minima of the hyperparameter optimization problem is computed by taken the minimum reached by 10 randomly initialized instances of \HOAG\ with exponential decrease tolerance.

The plot shows the relative accuracy of the different variants as a function of time. It can be seen that non-exact methods feature a cheaper iteration cost, yielding a faster convergence overall. Note that despite the global convergence result of Theorem~\ref{thm:convergence}, \HOAG\ is not guaranteed to be monotonically decreasing, and in fact, some degree of oscillation is expected when the decrease in the tolerance does not match the convergence rate (see e.g.~\citet{schmidt2011convergence}). This can be appreciated in Figure~\ref{fig:bench_tolerance}, where the quadratic decrease sequence (and to some extent the cubit too) exhibits oscillations in the two first plots.

\subsection{Comparison with other hyperparameter optimization methods}

We now compare against other hyperparameter optimization methods. The methods against which we compare are:

\begin{itemize}
\item {\bf \HOAG}. The method we present in this paper, with an exponentially decreasing tolerance sequence. A Python implementation is made freely available at {\small \url{https://github.com/fabianp/hoag}}.
\item {\bf Grid Search}. This method consists simply in splitting the domain of the hyperparameter into an equally-spaced grid. We split the interval $[-12, 12]$ into a grid of $10$ values.
\item {\bf Random}. This is the random search method~\citep{bergstra2012random} samples the hyperparameters from a predefined distribution. We choose to samples from a uniform distribution in the interval $[-12, 12]$.
\item {\bf SMBO}. Sequential model-based optimization using Gaussian Process. We used the implementation found in the Python package BayesianOptimization ({\small \url{http://github.com/fmfn/BayesianOptimization/}}). As initialization for this method, we choose $4$ values equally spaced between $-12$ and $12$. The acquisition function used is the expected improvement.
\item {\bf Iterdiff}. This is the iterative differentiation approach from~\citep{domke2012generic}, using the same inner-optimization algorithm as \HOAG. While the original implementation used to have a backtracking line search procedure to estimate the step size, we found that this performed worst than any of the alternatives. For this reason, we use the adaptive step size strategy presented in Section~\ref{scs:experiments} (assuming a zero tolerance parameter $\varepsilon$). 
\end{itemize}

\begin{figure*}\tabularnewline
\includegraphics[width=.33\linewidth]{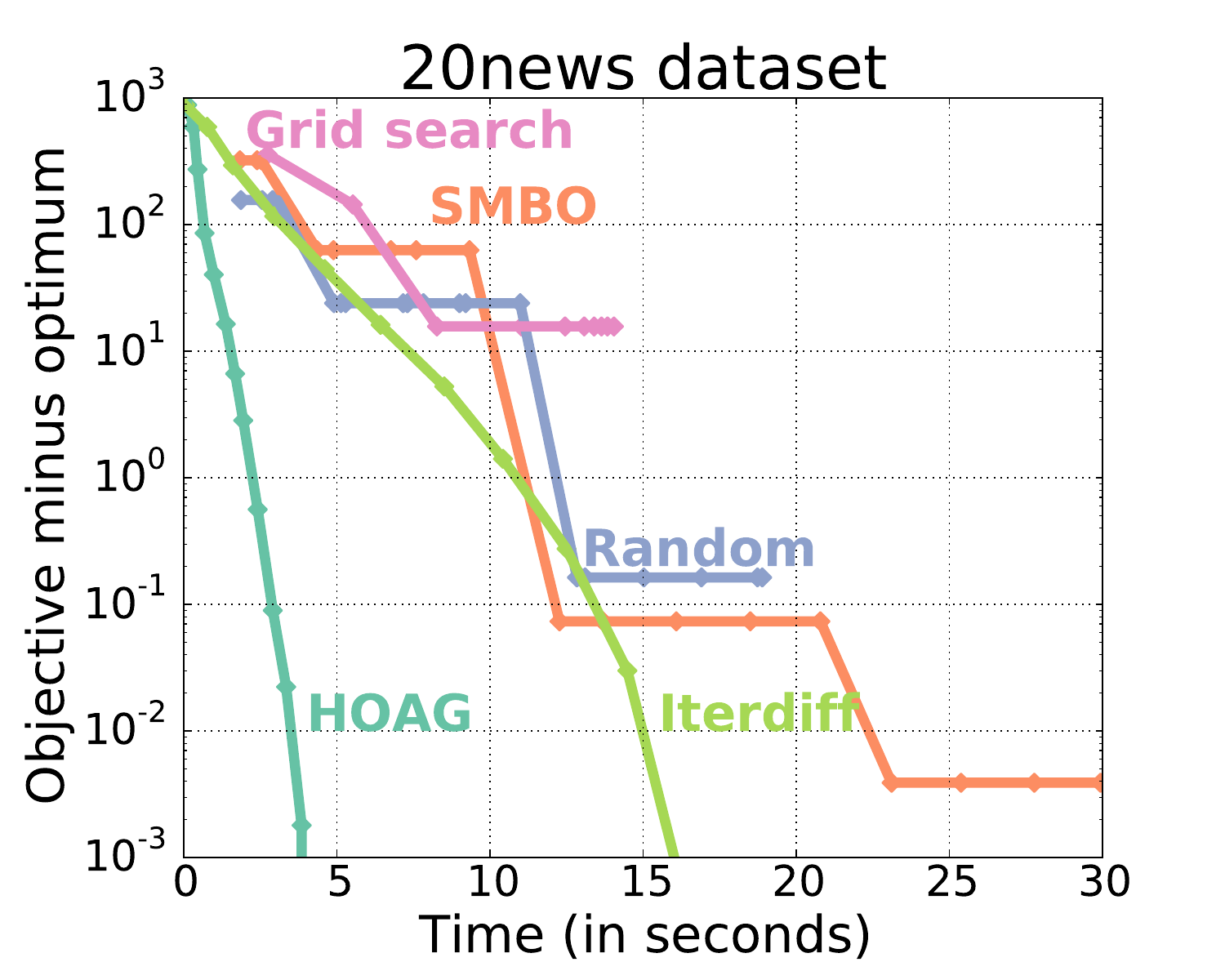}
\includegraphics[width=.33\linewidth]{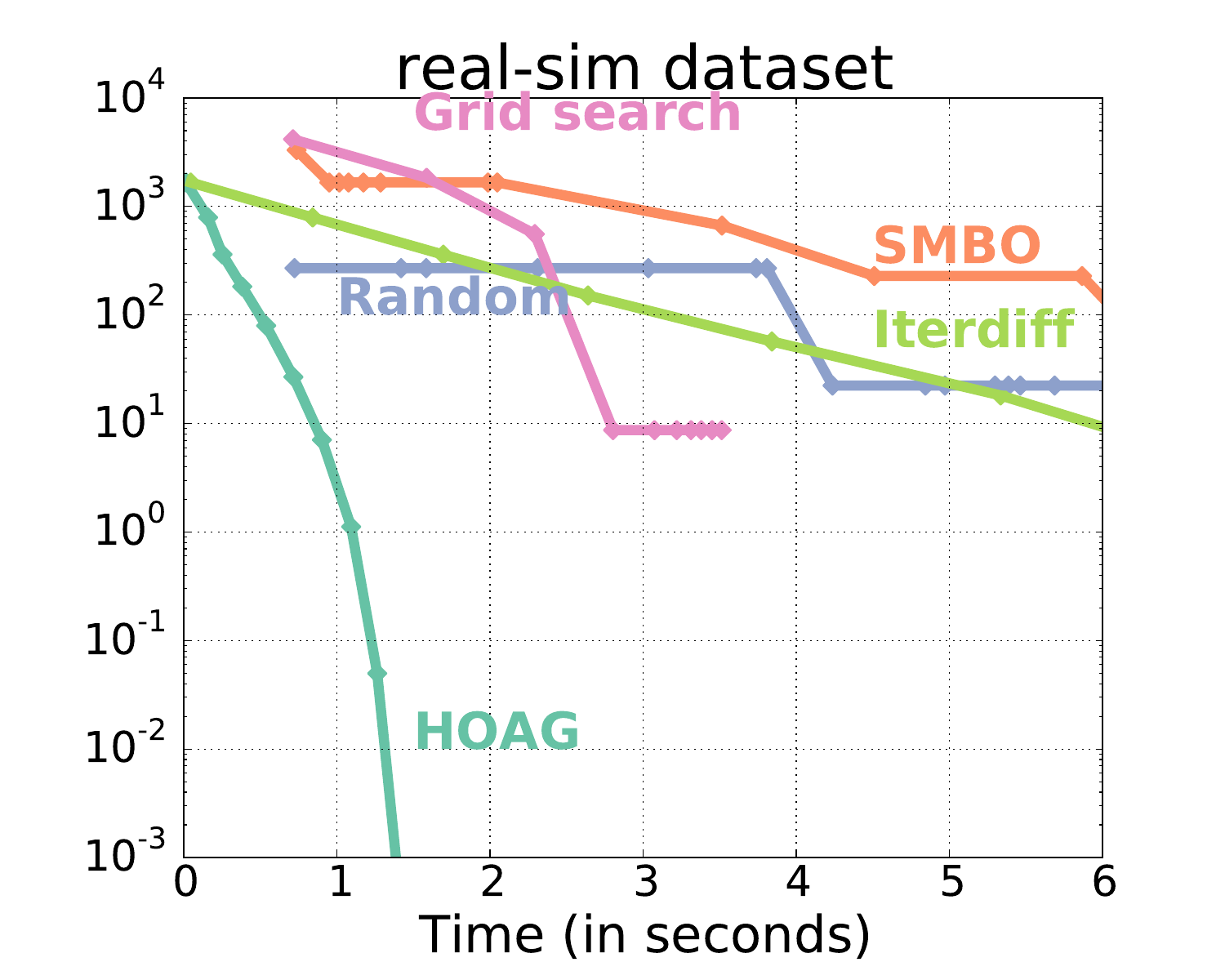}
\includegraphics[width=.33\linewidth]{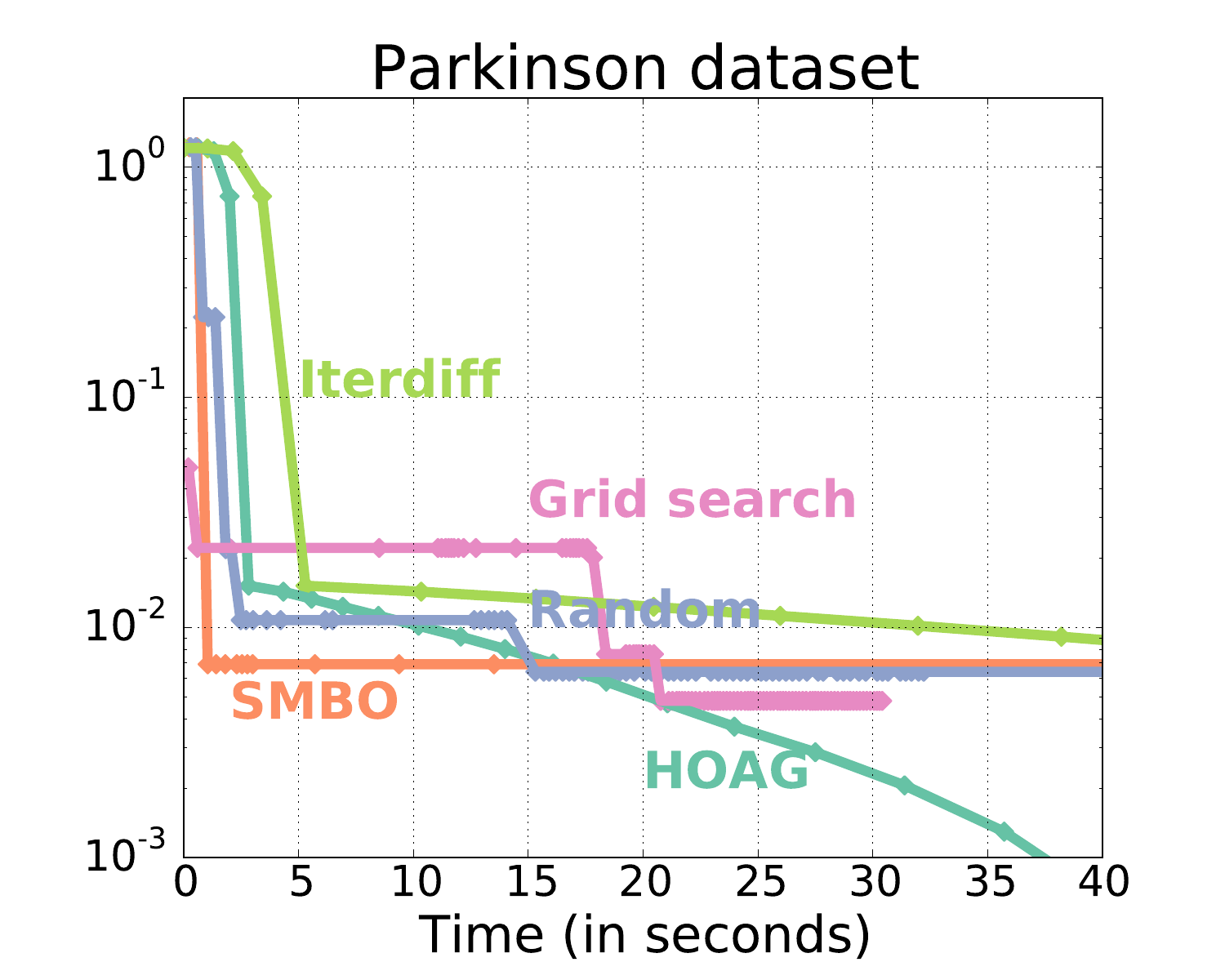}
\includegraphics[width=.33\linewidth]{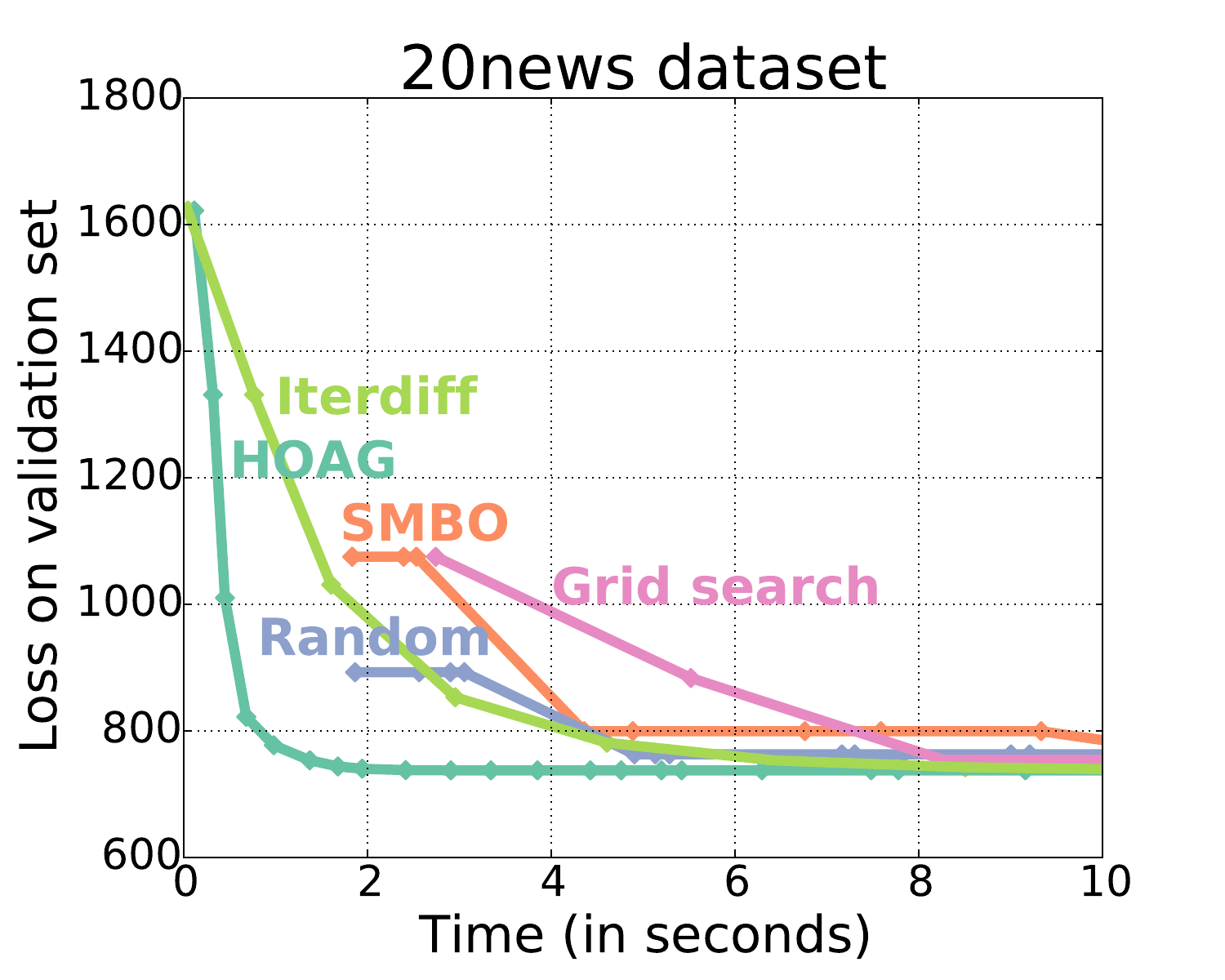}
\includegraphics[width=.33\linewidth]{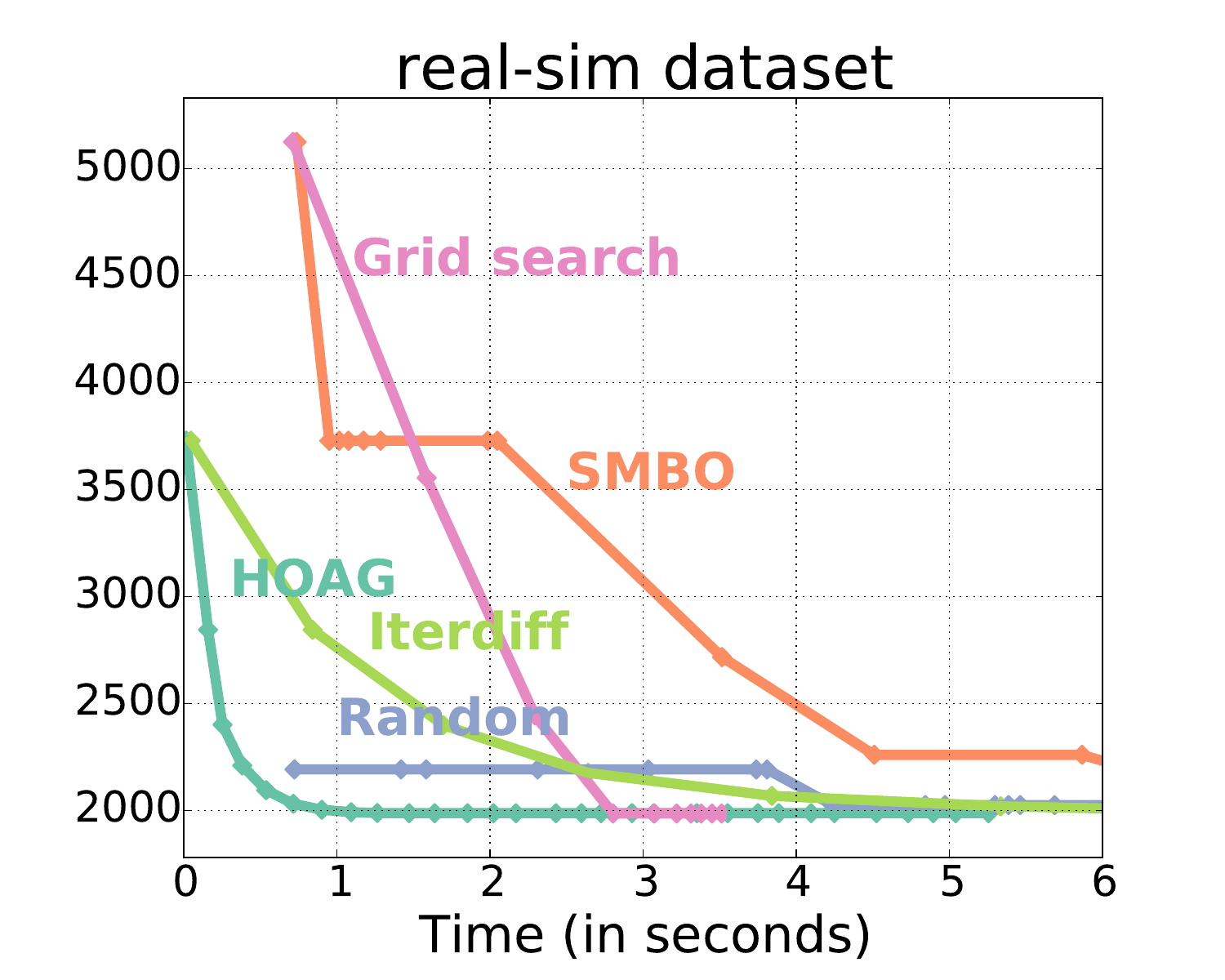}
\includegraphics[width=.33\linewidth]{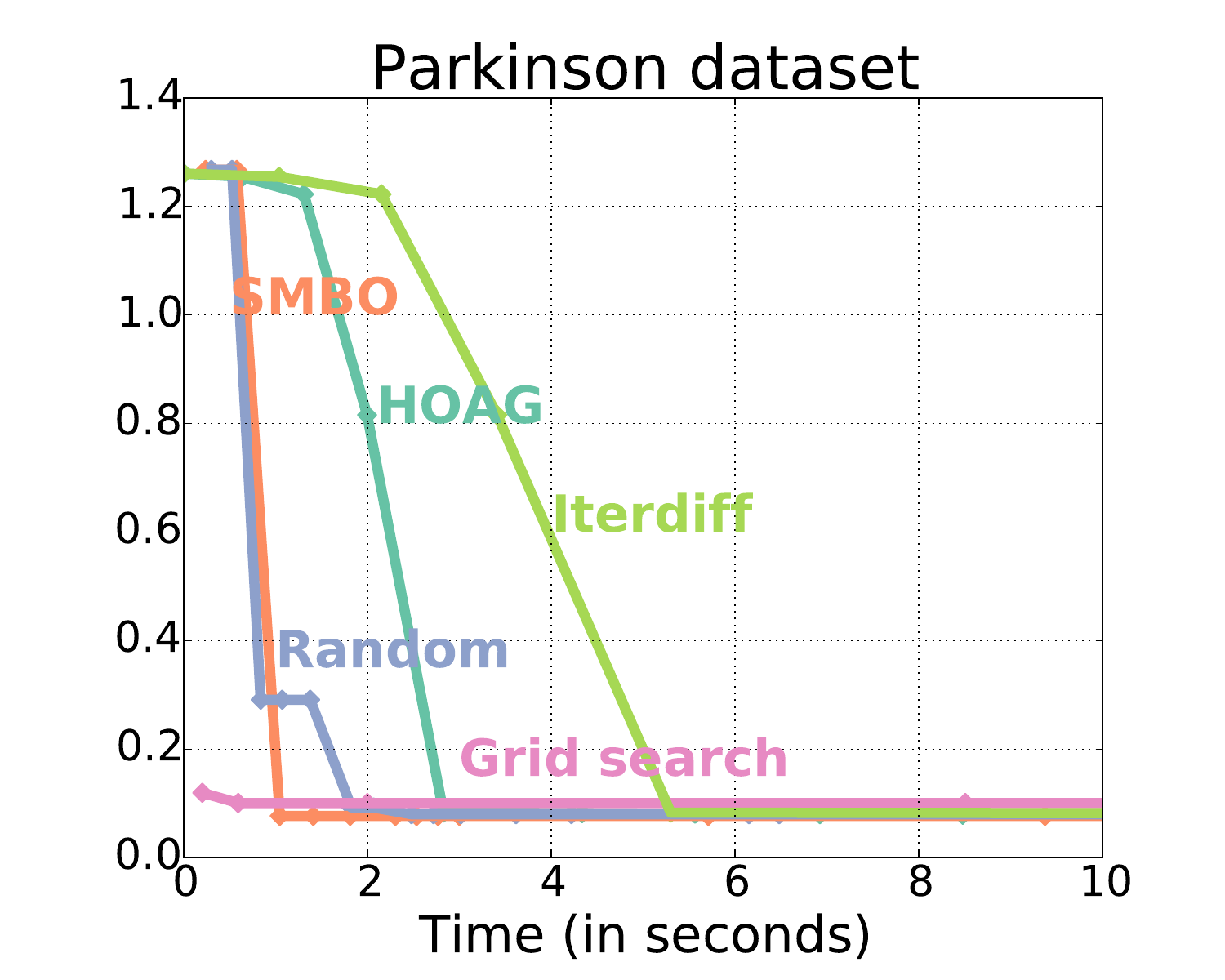}
\caption{{\bf Hyperparameter optimization methods}. Top row: suboptimality of the different methods in terms of the test loss. Bottom row: loss measured on a validation set for the different methods.}\label{fig:bench_ho}
\end{figure*}

For all methods, the number of iterations used in the inner optimization algorithm (L-BFGS or GD) is set to 100, which is the same used by the other methods and the default in the scikit-learn ({\small \url{http://scikit-learn.org}}) package.

We report in Figure~\ref{fig:bench_ho} the results of comparing the accuracy of these methods as a function of time. Note that it is expected that the different methods have different starting points. This is because Grid Search and SMBO naturally start from a pre-defined grid that starts from the extremes of the interval, while random search simply chooses a random point from the domain. For \HOAG\ and Iterdiff, we take the initialization $\lambda_1=0$.

In the upper row of Figure~\ref{fig:bench_ho} we can see the suboptimality of the different procedures as a function of time. We observe that \HOAG\ and Iterdiff have similar behavior, although \HOAG\ features a smaller cost per iteration. This can be explained because once \HOAG\ has made a step it can use the previous solution of the inner optimization problem as a \emph{warm-start} to compute the next gradient. This is not the case in Iterdiff since the computation of the gradient relies crucially on having sufficient iterations of the inner optimization algorithm. 

We note that in the Parkinson dataset, solution is inside a region that is almost flat (the different cost functions can be seen in Figure 1 of the supplementary material). This can explain the difficulty of the methods to go beyond the $10^{-2}$ suboptimality level. In this case, SMBO, who starts by computing the cost function at the extremes of the domain converges instantly to this region, which explains its fast convergence, although it is unable to improve the initially reached suboptimality.

Suboptimality plots are a standard way to compare the performance of different optimization methods. However, for the context of machine learning it can be argued that estimating hyperparameters up to a high precision is unimportant and that methods should be compared in terms of generalization performance. In the lower row of Figure~\ref{fig:bench_ho}, we display the test loss ($g$) on a validation set, that is, using a third set of samples $\{(\tilde{b}_i, \tilde{a}_i)\}_{i=1}^r$ which is different from both the train and test set. This figure reveals two main effects. First, unsurprisingly, optimization beyond $10^{-2}$ of relative suboptimality is not reflected in this metric. Second, the fast (but noisy) early iterations of \HOAG\ achieve the fastest convergence in two out of three datasets.


\section{Discussion and future work}\label{scs:discussion}


In previous sections we have presented and discussed several aspects of the \HOAG\ algorithm. Finally, we outline some future directions that we think are worthwhile exploring.

Given the success of recent stochastic optimization techniques~\citep{schmidt2013minimizing,johnson2013accelerating} it seems natural to study a {\bf stochastic variant} of this algorithm, that is, one in which the updates in the inner and outer optimization schemes have a cost that is independent of the number of samples. However, the dependency on the Hessian of the inner optimization ($\nabla^2_1 h$) in the implicit equation~\eqref{eq:grad_f_full} makes this non-trivial. 

Little is known of the {\bf structure of solutions} for the hyperparameter optimization problem~\eqref{eq:general_ho}. In fact, assumption (A3) is introduced almost exclusively in order to guarantee existence of solutions. At the same time recent progress on the setting of image restoration, which can be considered a subproblem of (HO), has given sufficient conditions on the input data for such solution to exist in an unbounded domain~\citep{reyes2015structure}. The characterization of solutions for the HO problem can potentially simplify the assumptions made in this paper. 



The analysis presented in this paper can be extended in several ways. For instance, the analysis of \HOAG\ is provided for a constant step size and not for the {\bf adaptive step size strategy} used in the experiments. Also, we have focused on proving asymptotic convergence of our algorithm. An interesting future direction would be to study {\bf rates of convergence}, which might give insight into an optimal choice for the tolerance decrease sequence.

Although we found the method to be quite robust in practice, there are situations where it can get stuck in flat regions. For example, if the initial step is too big, it might land in a region with a large regularization parameter where the curvature is amost zero (hence the reason to normalize the first step by its norm). An interesting direction of future work is to make the method robust to such flat regions, scaping from flat regions and allowing the method to make bigger steps in early iterations.



\subsection*{Acknowledgments}

I am in debt with Gabriel Peyr\'e for numerous discussions, suggestions and pointers. I would equally like to thank the anonymous reviewers for many insightful comments, and to Justin Domke for posting the code of his Iterative differentiation method. {\color{mydarkred}{\bf Update 2022.11}: I'm grateful to Clarice Poon for pointing an error in the proof of Theorem 2 in October 2022.}

Feedback and comments are welcome at the author's blog ({\small \url{http://goo.gl/WoV8R5}}).

The author acknowledges financial support from the ``Chaire Economie
des Nouvelles Donn\'ees'', under the auspices of Institut Louis Bachelier, Havas-Media and
Universit\'e Paris-Dauphine (ANR 11-LABX-0019).

\bibliography{biblio}{}

\clearpage
\appendix

\title{Appendix}
\date{}
\maketitle

\section{Analysis} \label{scs:appendix_analysis}

\begin{lemma}[The gradient error is bounded]
For sufficiently large $k$, the error in the gradient is bounded by a constant factor of $\varepsilon_k$. That is,
$$
\norm{\nabla f(\lambda_k) - p_k} = \mathcal{O}(\varepsilon_k) \quad.
$$
\end{lemma}
\begin{proof}
Before starting the proof, we introduce the following notation for convenience. We denote by $A_k$ and $\hat{A}_k$ the Hessian of the inner optimization function evaluated at the model parameters and at the $k$-th iteration approximation, respectively. That is,
$$
A_k = \nabla^2_1 h(\lambda_k, X(\lambda_k)) \quad\text{ and }\quad \hat{A}_k = \nabla^2_1 h(\lambda_k, x_k)\quad.
$$
In a similar way we define $b_k, \hat{b_k}$ and $D_k, \hat{D}_k$ to be the gradient of the outer loss and cross derivatives of the inner loss, respectively:
\begin{equation*}
\begin{split}
b_k = \nabla_1 g(\lambda_k, X(\lambda_k)) \quad\text{ and }\quad \hat{b}_k = \nabla_1 g(\lambda_k, x_k)\\
D_k = \nabla^2_{1, 2} h(\lambda_k, X(\lambda_k)) \quad \text{ and } \quad \hat{D}_k = \nabla^2_{1, 2} h(\lambda_k, x_k)
\end{split}
\end{equation*}

Note that $\hat{A_k}, \hat{b}_k$ and $\hat{D}_k$ are the quantities involved in the \HOAG\ algorithm. On the other hand, the quantities ${A_k}, {b}_k$ and ${D}_k$ are an ``ideal'' version of the former, computed when the tolerance $\varepsilon_k$ is zero. It is not surprising though that the difference between ${A_k}, {b}_k, {D}_k$ and its hat counterpart will play a fundamental role in the proof.

The proof is structured in two parts. In part (i) we will prove that several sequences of interest are bounded, while in part (ii) we will use this to prove the main result.

\hfill

\underline{Part (i)}. We first note that that both $\norm{\lambda}$ and $\norm{X(\lambda)}$ are bounded and denote such bounds by $\xi$ and $\eta$, respectively. $\norm{\lambda}$ is bounded as a direct consequence of assumption (A3). On the other hand, $X(\lambda)$ is continuously differentiable as a result of the implicit function theorem. Since its domain is a bounded set, $\|X(\lambda)\|$ is also bounded. We prove that the following sequences are bounded:
\begin{itemize}
\item $\{\norm{x_k}\}_{k=1}^{\infty}$. By the termination condition of the inner optimization problem we have that ${\norm{X(\lambda_k) - x_k} \leq \varepsilon_k}$. Using the reverse triangular inequality we further have
\begin{equation*}
\begin{split}
 \varepsilon_k \geq \norm{X(\lambda_k) - x_k} \geq \abs{\norm{X(\lambda_k)} - \norm{x_k}} \\ \implies \norm{x_k} \leq \norm{X(\lambda_k)} + \varepsilon_k \leq \zeta + \varepsilon_k \quad
\end{split}
\end{equation*}
  hence $\norm{x_k}$ is a bounded sequence since $\{\varepsilon\}_{k=1}^\infty$ defines a summable (hence bounded) sequence.
\item $\{\norm{A_k}\}_{k=1}^{\infty}$, $ \{\|{D_k}\|\}_{k=1}^{\infty}$ and $\{\|{b_k}\|\}_{k=1}^{\infty}$. Assumption (A1) implies that there exists constants $L_E, L_g$ such that
$$
\begin{aligned}
\norm{A_k - A_0} &\leq L_E \norm{[\lambda_k, X(\lambda_k)] - [\lambda_0, X(\lambda_0)]} \\
&\leq L_E \sqrt{\xi^2 + \eta^2}\\
\norm{D_k-D_0} &\leq L_E \norm{[\lambda_k, X(\lambda_k)] - [\lambda_0, X(\lambda_0)]} \\ &\leq L_E \sqrt{\xi^2 + \eta^2}\\
\norm{b_k - b_0} &\leq L_g \norm{[\lambda_k, X(\lambda_k)] - [\lambda_0, X(\lambda_0)]} \\ &\leq L_g \sqrt{\xi^2 + \eta^2}\\,
\end{aligned}
$$
and so $\norm{A_k}, \norm{D_k}$ and $\norm{b_k}$ are all bounded sequences. 
\item $\{\|{A_k^{-1}}\|\}_{k=1}^{\infty}$. By assumption (A2), $A_k^{-1}$ exists and $\|A^{-1}_k\| < \infty$ for all $k$. Since $\DD$ is a compact set, the limit of the sequence verifies  $\lim_{k\to \infty}\|A^{-1}_k\| < \infty$, and hence is bounded.

\end{itemize}

\underline{Part (ii)}. By the Lipschitz assumption on $g$ and $h$, there exists finite numbers $L_g$ and $L_E$ such that we have the following sequence of inequalities:
$$
\begin{aligned}
{\|A_k - \hat{A}_k\|} &\leq L_E {\|X(\lambda_k) - x_k\|} \leq  L_E \varepsilon_k \\
{\|b_k - \hat{b}_k\|} &\leq L_g {\|X(\lambda_k) - x_k\|} \leq  L_g \varepsilon_k
\end{aligned}
$$

Let $z_k$ be the solution to the linear system of equations $A_k z_k = b_k$ and $\hat{z}$ be the solution to $\hat{A}_k \hat{z_k} = \hat{b}_k$. Since by assumption $A_k$ is invertible, $z_k = A_k^{-1} b_k$ and by the Cauchy-Schwarz inequality, we have that
$$
\norm{z_k} \leq \|{A_k^{-1}}\| \norm{b_k} \quad,
$$
and hence the sequence $\norm{z_k}$ is also bounded.

By the summability condition of $\varepsilon_k$ and the boundedness of $\|A^{-1}_k\|$ proved in part (i) of this proof, for all sufficiently large $k$ it is verified that $\varepsilon_k \|A_k^{-1}\| L_g \leq \rho < 1$ and by classical results related to the sensitivity of linear systems (see e.g. \citep[\S 7.1]{higham2002accuracy}) we have the following sequence of inequalities:
\begin{equation}\label{eq:bound_z}
\begin{aligned}
{\norm{z_k - \hat{z}_k}} &\leq \frac{\varepsilon_k}{ 1 - \varepsilon_k\|{A_k^{-1}}\| L_g} \left({\|{A_k^{-1}}\|L_f} + {\norm{z_k}} \|A_k^{-1}\|L_g\right) \\
&\leq \frac{\varepsilon_k}{ 1 - \rho } \left({\|{A_k^{-1}}\|L_f} + {\norm{z_k}} \|A_k^{-1}\|L_g\right)  \\
&= \mathcal{O}(\varepsilon_k) \qquad \qquad \text{ (bound on all terms involved)}
\end{aligned}
\end{equation}
where the first inequality is derived from~\citep[\S 7.1]{higham2002accuracy} and the second one comes from the definition of $\rho$.

This last equation provides a bound on the difference between solving the linear system at $X(\lambda)$ and solving the linear system at $x_k$, assuming that the linear system is solved to full precision. However, in practice we do not attempt to solve the linear system to full precision but rather compute an approximation $q_k$ such that
$
\|{\hat{A}_k q_k - \hat{b}_k}\| \leq \varepsilon_k$ { (step $(ii)$ of \HOAG)}
and we are interest in bounding the distance between $q_k$ and its noise-less version $z_k$. We have the following sequence of inequalities:
$$
\begin{aligned}
&\norm{q_k - z_k} = \|{q_k - \hat{z}_k + \hat{z}_k - z_k}\|  \\
&\leq \|q_k - \hat{z}_k\| + \|\hat{z}_k - z_k\| \qquad \text{(triangular inequality)} \\
&\leq \|\hat{A}^{-1}_k \hat{A}_k (q_k - \hat{z}_k)\| + \|\hat{z}_k - z_k\| \\
&= \|\hat{A}^{-1}_k (\hat{A}_k q_k - \hat{b}_k)\| + \|\hat{z}_k - z_k\| \quad \text{(definition of $\hat{z}_k$)} \\
&\leq \|\hat{A}^{-1}_k\| \|\hat{A}_k q_k - \hat{b}_k\| + \|\hat{z}_k - z_k\| \quad \text{(Cauchy-Schwarz)} \\
&\leq \|\hat{A}^{-1}_k\| \varepsilon_k  + \|\hat{z}_k - z_k\| \qquad \text{ (definition of $q_k$)} \\
&= \mathcal{O}(\varepsilon_k) \qquad \text{ (Eq.~\eqref{eq:bound_z})} \quad.
\end{aligned}
$$
Finally, using this we can write that the difference between $p_k$ and the true gradient. Let $c_k, \hat{c}_k$ be defined as
\begin{equation*}
\begin{split}
c_k = \nabla_2 {g}(\lambda_k, X(\lambda_k)) \quad \text{ and } \quad \hat{c}_k = \nabla_2 {g}(\lambda_k, x_k)  \\
\end{split}
\end{equation*}
Then it is verified that
$$
\begin{aligned}
&\norm{\nabla f(\lambda_k) - p_k} = \|{c_k - D^T z_k - \hat{c}_k - \hat{D}_k^T q_k}\| \\ 
&\leq \|c_k - \hat{c}_k\| + \|D_k^T z_k - \hat{D}_k^T q_k\|  \\& \qquad\qquad \text{(triangular inequality)} \\
&\leq \|c_k - \hat{c}_k\| + \|D_k^T z_k - \hat{D}_k^T z_k + \hat{D}_k^T z_k - \hat{D}_k^T q_k\| \\& \qquad\qquad \text{(Add and remove $\hat{D}_k^T z_k$)} \\
&\leq \|c_k - \hat{c}_k\| + \|D_k^T z_k - \hat{D}_k^T z_k\| + \|\hat{D}_k^T z_k - \hat{D}_k^T q_k\| \\& \qquad\qquad \text{(triangular inequality)} \\
&\leq \|c_k - \hat{c}_k\| + \|D_k - \hat{D}_k\|\|z_k\| + \|\hat{D}_k\|\|z_k - q_k\| \\& \qquad\qquad \text{(Cauchy-Schwartz)} \\
&\leq L_g \varepsilon_k + L_E \varepsilon_k \|z_k\| + \|\hat{D}_k\|\|z_k - q_k\|  \\ & \qquad\qquad \text{(Assumption (A1))} \\
&\leq L_g \varepsilon_k + L_E \varepsilon_k \|z_k\| + \|\hat{D}_k\|\mathcal{O}(\varepsilon_k ) \\&\qquad \qquad \text{(previous inequality)} \\
&= \mathcal{O}(\varepsilon_k)\qquad \text{(bound on $\|z_k\|$ and $\|\hat{D}_k\|$)}
\end{aligned}
$$
which completes the proof.
\end{proof}

\begin{customthm}{2}[Global convergence] If the tolerance sequence is summable, that is, if $\{\varepsilon\}_{i=1}^n$ is positive and verifies
$$
\sum_{i=1}^\infty \varepsilon_i < \infty \quad,
$$
then the limit of the \HOAG\ iterates verifies the stationary point condition:
$$
\lim_{k \to \infty} \|\lambda_k - P_{\DD}\big(\lambda_k - \frac{1}{L} \nabla f(x_k)\big)\| = 0\,.
$$
Furthermore, if the iterates belong to the domain for sufficiently large $k$, then
$$
\lim_{k \to \infty} \|\nabla f(x_k)\| = 0\,.
$$
\end{customthm}
\begin{proof}
By assumption, $f$ is $L$-smooth. This implies the following inequality for any pair of values $\alpha, \beta \in \DD$:
\begin{equation}\label{eq:lipschitz}
f(\beta) \leq f(\alpha) + \nabla f(\alpha)^T (\beta - \alpha) + \frac{L}{2} \|\beta - \alpha\|^2 \quad.
\end{equation}
This is a classical result on quadratic upper bounds for \mbox{$L$-smooth} functions (see e.g.~\citep[Lemma 1.2.3]{nesterov2004introductory}).
We will also make use of the following inequality concerning the projection $P_\DD$, which stems from the fact that projections onto convex sets are \emph{firmly nonexpansive} operators (see e.g.~\citet[\S 2.2]{parikh2013proximal}). Let $\eta, \nu \in \RR^s$, then the following is verified:
\begin{equation}\label{eq:nonexpansiveness}
\|P_\DD(\eta) - P_\DD(\nu)\|^2 \leq (\eta - \nu)^T (P_\DD(\eta) - P_\DD(\nu))
\end{equation}
In particular, for $\eta = \lambda_k, \nu = \lambda_k - \frac{1}{L} p_k$, this reduces to
\begin{equation}\label{eq:nonexpansive}
\|\lambda_k - \lambda_{k+1}\|^2 \leq \frac{1}{L}p_k^T (\lambda_k - \lambda_{k+1})
\end{equation}
Setting now $\alpha = \lambda_k, \beta = \lambda_{k+1} = P_\DD(\lambda_k - \frac{1}{L}p_k)$ in Eq.~\eqref{eq:lipschitz}, we have the following sequence of inequalities
\begin{equation}\label{eq:descent_ineq}
\begin{aligned}
f(\lambda_{k+1}) &\leq f(\lambda_k) - \nabla f(\lambda_k)^T (\lambda_k - \lambda_{k+1})  \\
&\qquad \qquad +\frac{L}{2}\norm{\lambda_{k+1} - \lambda_k}^2 \\
&= f(\lambda_k) - (\nabla f(\lambda_k) - p_k + p_k)^T (\Delta \lambda_k - \lambda_{k+1}) \\
&\qquad\qquad + \frac{L}{2}\norm{\lambda_{k+1} - \lambda_k}^2  \\
&= f(\lambda_k) - (\nabla f(\lambda_k) - p_k)^T (\lambda_k - \lambda_{k+1}) \\
&\quad -  p_k^T (\lambda_k - \lambda_{k+1}) + \frac{L}{2}\norm{\lambda_{k+1} - \lambda_k}^2 \\
&\leq f(\lambda_k) - (\nabla f(\lambda_k) - p_k)^T (\lambda_k - \lambda_{k+1})\\&\qquad - \frac{L}{2}\norm{\lambda_{k+1}  - \lambda_k}^2 \qquad \text{ (by Eq.~\eqref{eq:nonexpansive})}\\
&\leq f(\lambda_k) + \|\nabla f(\lambda_k) - p_k\| \|\lambda_k - \lambda_{k+1}\|  \\& \qquad - \frac{L}{2}\norm{\lambda_{k+1} - \lambda_k}^2 \qquad\text{ (Cauchy-Schwartz)}\\
\end{aligned}
\end{equation}

By Lemma~\ref{thm:bound_gradient}, $\|\nabla f(\lambda_k) - p_k\| = \mathcal{O}(\varepsilon_k)$. Since $\DD$ is bounded, we have that there exists $M> 0$ and such that for sufficiently large $k$
\begin{equation}
\|\nabla f(\lambda_k) - p_k\| \|\lambda_k - \lambda_{k+1}\| < M \varepsilon_k \quad.
\end{equation}
which applied to the previous inequality results in
$$
f(\lambda_{k+1}) \leq f(\lambda_k) + M \varepsilon_k - \frac{L}{2}\norm{\lambda_{k+1} - \lambda_k}^2 \quad,
$$
or equivalently
$$
\norm{\lambda_{k+1} - \lambda_k}^2 \leq \frac{2}{L} \left(f(\lambda_k) - f(\lambda_{k+1})  + M \varepsilon_k \right) \quad.
$$

Let $C$ be an lower bound on the function $f$. This bound exist and is finite because $f$ has continuous derivatives and is defined on a compact set. Summing the last expression from $k=m$ to $k=\infty$ we obtain
$$
\sum_{k = m}^\infty \norm{\lambda_{k+1} - \lambda_k}^2 \leq \frac{2}{L} \left(f(\lambda_m) - C  + M \sum_{k=m}^\infty \varepsilon_k \right) \quad.
$$
Since $\{\varepsilon_k\}_{k=1}^\infty$ is a summable sequence we conclude that the right-hand side of this expression is finite. Hence, for the sum on the left-hand side to be finite we must have $\lim_{k \to \infty} \norm{\lambda_{k+1} - \lambda_k}^2 = 0$. Replacing $\lambda_{k+1}$ by its definition, we have
\begin{equation}\label{eq:limitpk}
\lim_{k \to \infty}\|\lambda_k - P_\DD(\lambda_k - \tfrac{1}{L}p_k) \| = 0
\end{equation}
This is almost the result that we want, except with $p_k$ instead of the gradient. To show that this result holds when the $p_k$ is replaced by the true gradient, we will use the non-expansiveess of the prox and existing bounds between the gradient and $p_k$. More precisely, for the gradient mapping, we have
\begin{align}
&\lim_{k \to \infty}\|\lambda_k - P_\DD(\lambda_k - \tfrac{1}{L}\nabla f(\lambda_k)) \| \\
&\stackrel{\text{(triangular)}}{\leq} 
\lim_{k \to \infty}\|P_\DD(\lambda_k - \tfrac{1}{L}p_k) - P_\DD(\lambda_k - \tfrac{1}{L}\nabla f(\lambda_k))\| \nonumber\\
&\qquad+ \underbrace{\lim_{k \to \infty}\|P_\DD(\lambda_k - \tfrac{1}{L} p_k) - \lambda_k\|}_{0 \text{ by \eqref{eq:limitpk}}} \\
&\stackrel{\text{\eqref{eq:nonexpansiveness}}}{\leq} \lim_{k \to \infty}\|\tfrac{1}{L}\nabla f(\lambda_k) - \tfrac{1}{L}p_k\|
\stackrel{\text{(Lemma \ref{thm:bound_gradient})}}{=} 0\,.
\end{align}
We've shown that the limit of the gradient mapping is upper bounded by zero. Since the norm is non-negative, it must be zero.

When the iterates belong to the domain, the projection is the identity and so the gradient mapping becomes the (scaled) gradient norm, from where we have $\lim_{k \to \infty} \|\nabla f(x_k)\| = 0\,.$
\end{proof}

\section{Experiments}\label{scs:appendix_experiments}

\begin{figure}
\center \includegraphics[width=0.7 \linewidth]{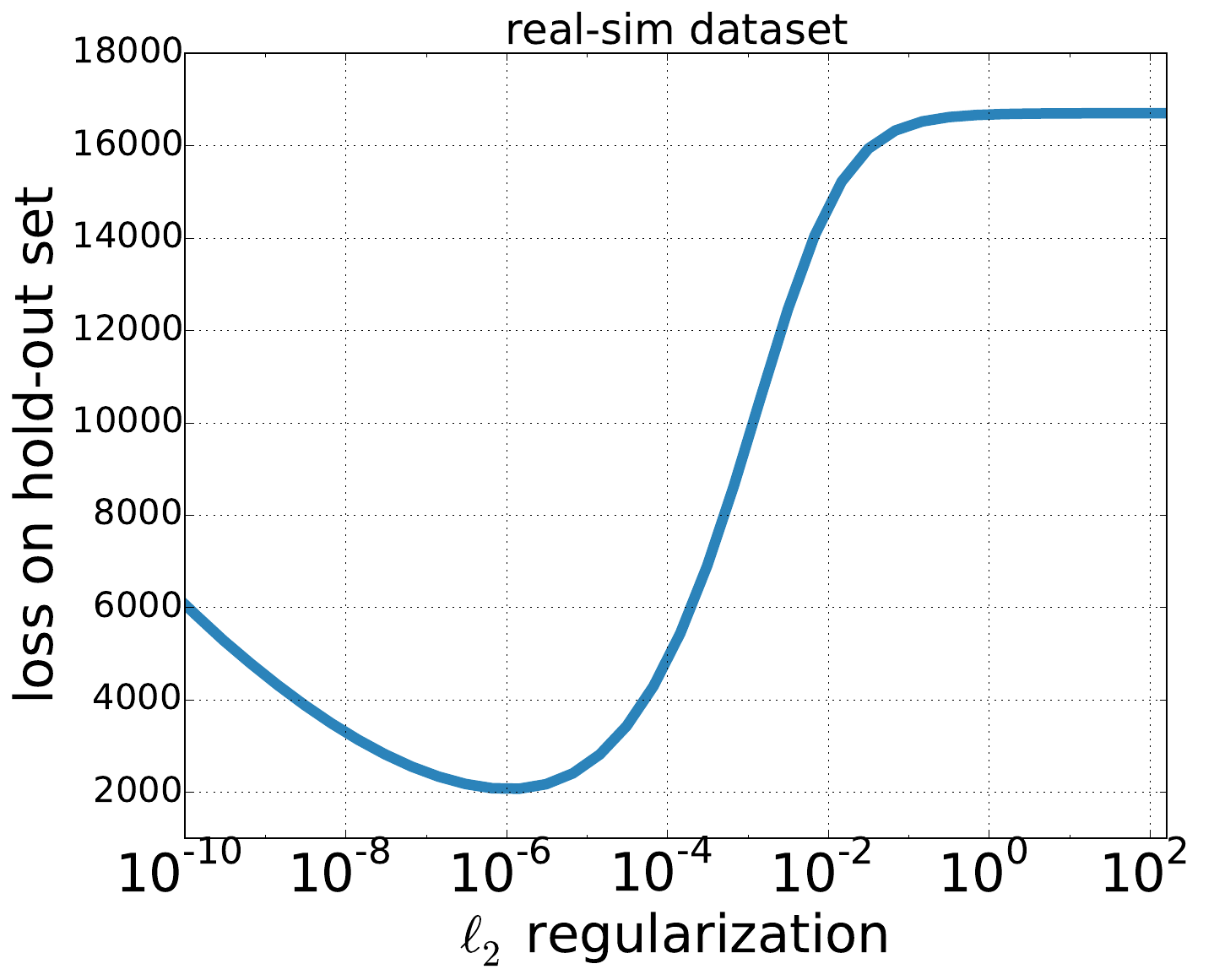}
\center \includegraphics[width=0.7 \linewidth]{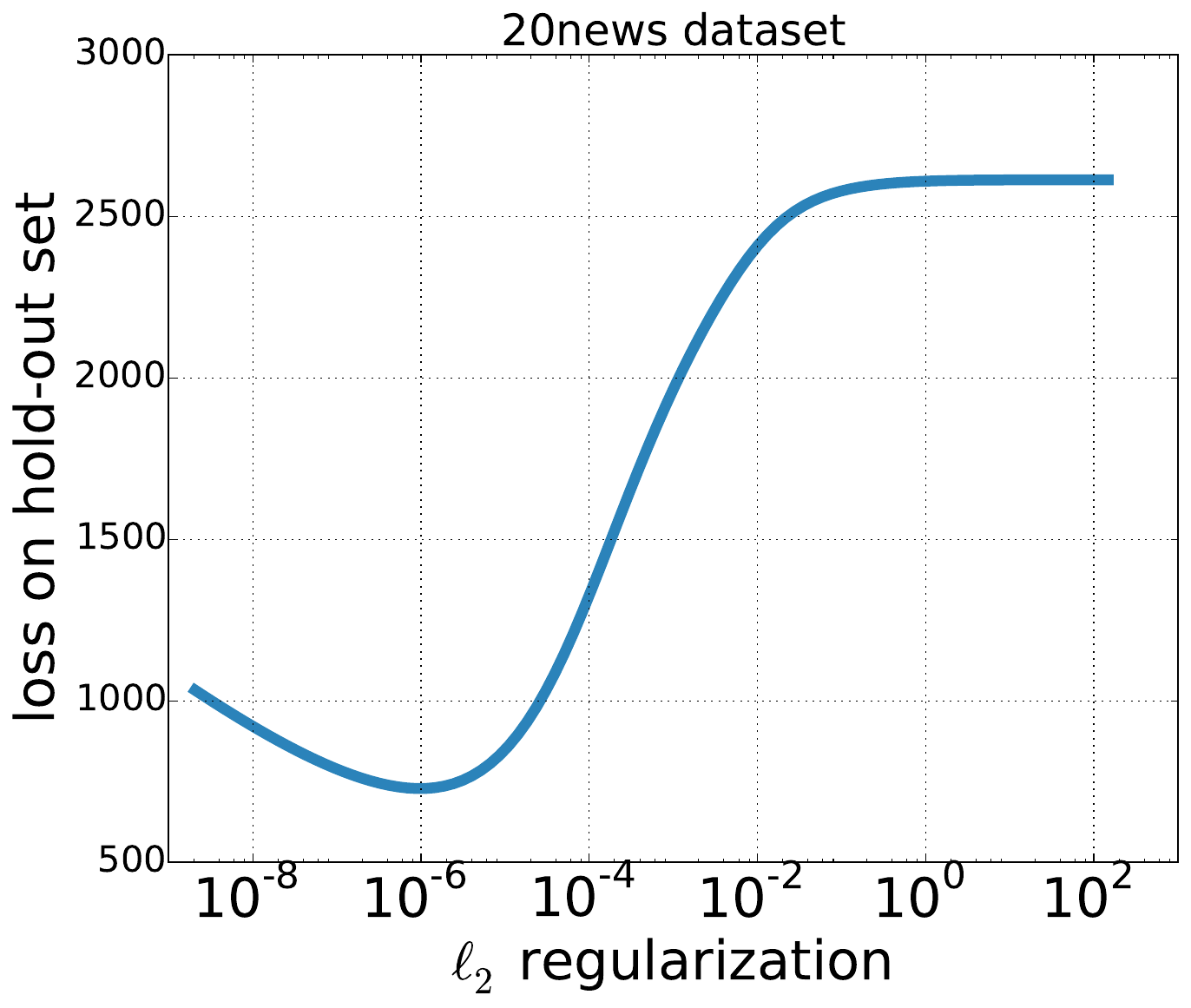}
\includegraphics[width=\linewidth]{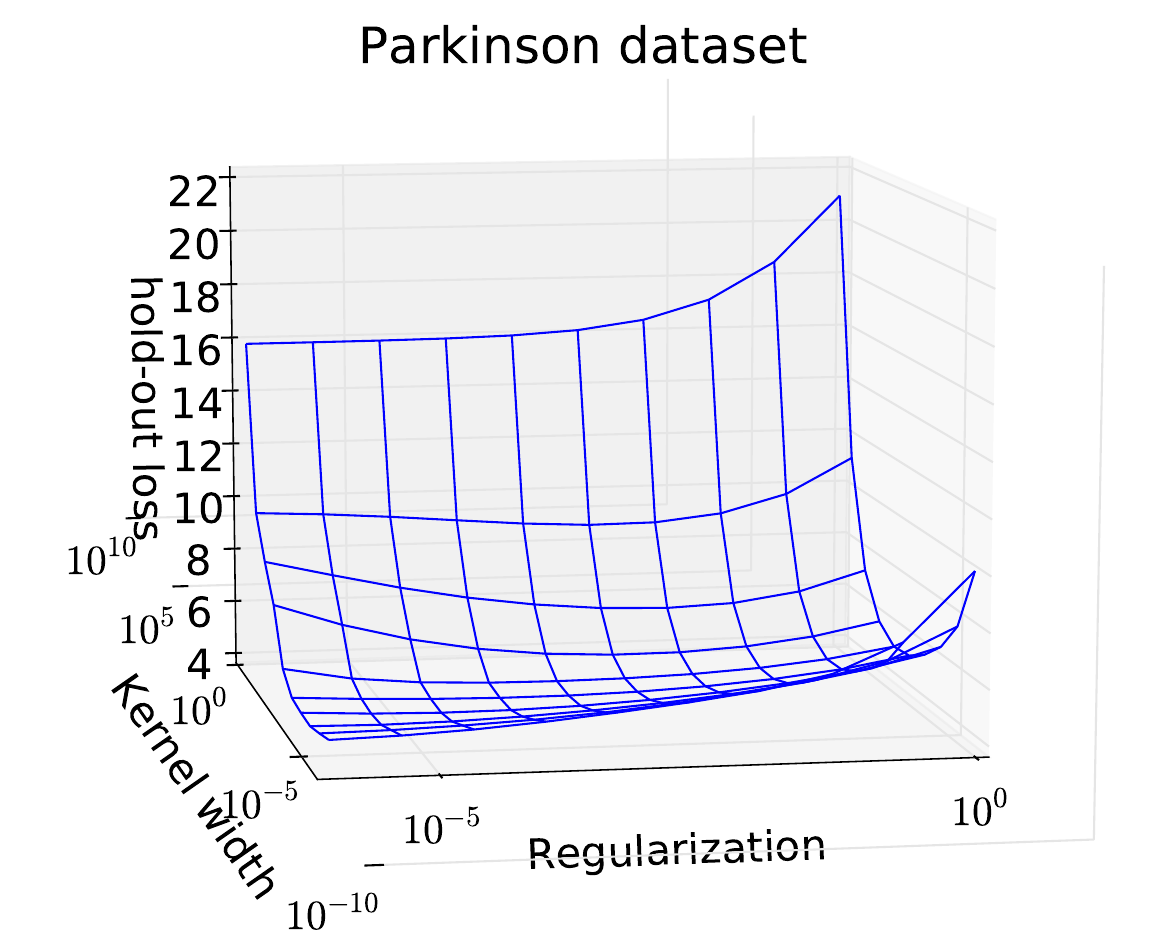}
\caption{{\bf Cost functions of the different hyperparameter optimization methods}. These are the cost functions (denoted $f$ through the paper) as a function of the hyperparameters ($\lambda$) fo the three problems considered. In the two first images, the hyperparameter is the $\ell_2$ regularization parameter and the objective function has a unique minima. The third features the cost function as a function of the kernel width and $\ell_2$ regularization.}
\end{figure}

\subsection*{Stopping criterion}

In order to bound the tolerance we used the following lemma
\begin{lemma} Let $s: \RR^p \to \RR$ be $\mu$-strongly convex and \mbox{$L$-smooth}. We denote by $x^* \in \RR^p$ the minimum of this function. Then the following is verified for all $x$ in the domain
$$
\|x^* - x\| \leq \mu^{-1} \|\nabla s(x)\|
$$
\end{lemma}
\begin{proof} Using basic properties of $\mu$-strongly convex functions  (see e.g.~\citep{nesterov2004introductory}), we have the following sequence of inequalities for all $x, y$ in the domain:
$$
\begin{aligned}
\mu \norm{x - y}^2 &\leq \langle \nabla s(x) - \nabla s(y), x - y \rangle \\
&\qquad \text{ (by strong convexity)} \\
&\leq \norm{\nabla s(x) - \nabla s(y)} \norm{x - y} \\
&\qquad \text{ (by Cauchy-Schwarz)} \\
\end{aligned}
$$
from where $\mu \norm{x - y} \leq \norm{\nabla s(x) - \nabla s(y)}$. Specializing at $y=x^*$ and using $s(x^*) = 0$ (the inner optimization is unconstrained) yields the desired result.
\end{proof}

\subsection*{Adaptive step size}

We will now derive the inequality used by the adaptive step size procedure presented in the Experiments section. The derivation of this procedure uses the \mbox{$L$-smooth} assumption and Theorem~\ref{thm:bound_gradient}. The $L$-smooth assumption implies that for all $\alpha, \beta$ in the domain, the following inequality is verified
\begin{equation*}
f(\beta) \leq f(\alpha) + \nabla f(\alpha)^T (\beta - \alpha) + \frac{L}{2} \|\beta - \alpha\|^2 \quad.
\end{equation*}
Setting $\alpha = \lambda_{k-1}, \beta = \lambda_k$ in the above inequality, and using the bound $\|\nabla f(\lambda_{k}) - p_k\| < M \varepsilon_k$ given by Theorem~\ref{thm:bound_gradient}, we have that
\begin{equation}\label{eq:line_search_0}
\begin{aligned}
f(\lambda_k) \leq &f(\lambda_{k-1}) + \varepsilon_{k-1} M \Delta_k - L\Delta_k^2 \quad.
\end{aligned}
\end{equation}
for some constant $M$ and where $\Delta_k$ is defined as $\norm{\lambda_k - \lambda_{k-1}}$ (A more rigorous derivation of this inequality can be found in the supplementary material, Eq. (4)). Now, we do not have access to $f(\lambda_k)$ as this depends on the exact model parameters. However, by by the definition of $x_k$ we have that $\norm{x_k - X(\lambda)} \leq \varepsilon_k$. Furthermore, $g$ is Lipschitz continuous since this is a weaker condition than assumption (A1), hence there exists a constant $C$ such that $\|g(\lambda_k, x_k) - g(\lambda_k, X(\lambda_k))\| \leq C \varepsilon_k$. Since by definition $f(\lambda) = g(\lambda, X(\lambda))$, we can derive the inequalities $g(\lambda_k, x_k) - C \varepsilon_k \leq f(\lambda_k)$ and $f(\lambda_{k-1}) \leq g(\lambda_{k-1}, x_k) + C \varepsilon_{k-1}$. Replacing this into Eq.~\eqref{eq:line_search_0} yields the following inequality:
\begin{equation}
\begin{aligned}
g(\lambda_k, x_k) \leq ~&g(\lambda_{k-1}, x_{k-1}) + C \varepsilon_k + \\
&\varepsilon_{k-1} (C + M) \Delta_k - L\Delta_k^2 \quad.
\end{aligned}
\end{equation}

\subsection*{Datasets}

The first dataset that we consider is the ``20news dataset\footnote{\url{http://qwone.com/~jason/20Newsgroups/}}'' which contains 18000 newsgroups posts on 20 topics, with the task of predicting the appropriate group of a post. The features we used are the tf-idf vectors obtained from the original dataset, and the groups were randomly split into two categories to obtain a binary classification problem.

The second dataset that we use is also a text categorization task denoted ``real-sim''. This dataset contains 73218 UseNet articles from four discussion groups, for simulated auto racing, simulated aviation, real autos, real aviation. The binary classification task is to predict whether it belongs to the real-\{autos, aviation\} group or the simulated-\{aviation, auto racing\} group. This dataset was obtained from the libsvmtools project\footnote{\url{http://www.csie.ntu.edu.tw/~cjlin/libsvmtools/datasets/binary.html}}.

The third dataset, which we denote the ``Parkinson dataset'', is used by the kernel Ridge regression problem. This dataset is composed of a range of biomedical voice measurements from 42 
people with early-stage Parkinson's disease\citep{tsanas2010accurate}. This dataset contains 5875 samples with 26 features and is publicly available from the UCI machine learning repository\footnote{\url{http://archive.ics.uci.edu/ml/}}.

The forth dataset we used is the MNIST\footnote{\url{http://yann.lecun.com/exdb/mnist/}} dataset, with 60000 samples and images subsampled to $12 \times 12$ pixels.  We replicate the model used in~\citep[\S 3.2]{MacDuvAda2015hyper}, where the authors consider a multinomial logistic regression model with one regularization parameter per feature, totaling $12 \times 12$ (number of features) $ \times 10$ (number of classes) $ = 1440$ hyperparameters. For completeness we state here the loss functions involved in this model. Denoting by $\psi$ the multinomial logistic loss, the inner optimization loss $h$ and its gradients read: 
$$
\begin{aligned}
h(\lambda, x) &= \sum_{i=1}^n \psi(b_i, a_i^T x) + \frac{1}{2} \sum_{j=1}^p e^{\lambda_j} x_j^2 \\
\nabla_{1} h &= \sum_{i=1}^n b_i a_i^T \psi'(b_i a_i^T x) + e^{\lambda} * x\\
\nabla_{2} h^2_{1, 2} &= \text{ diagonal matrix with elements } (e^{\lambda_i} x_i)_{i, i} \\
\psi(z) &= \text{multinomial logistic loss} \quad,
\end{aligned}
$$
where $*$ denotes component-wise multiplication and $e^\lambda = (e^{\lambda_1}, e^{\lambda_2}, \ldots, e^{\lambda_s})$.

\subsection*{Further experimental validation}
\begin{figure}
\center \includegraphics[width=.7\linewidth]{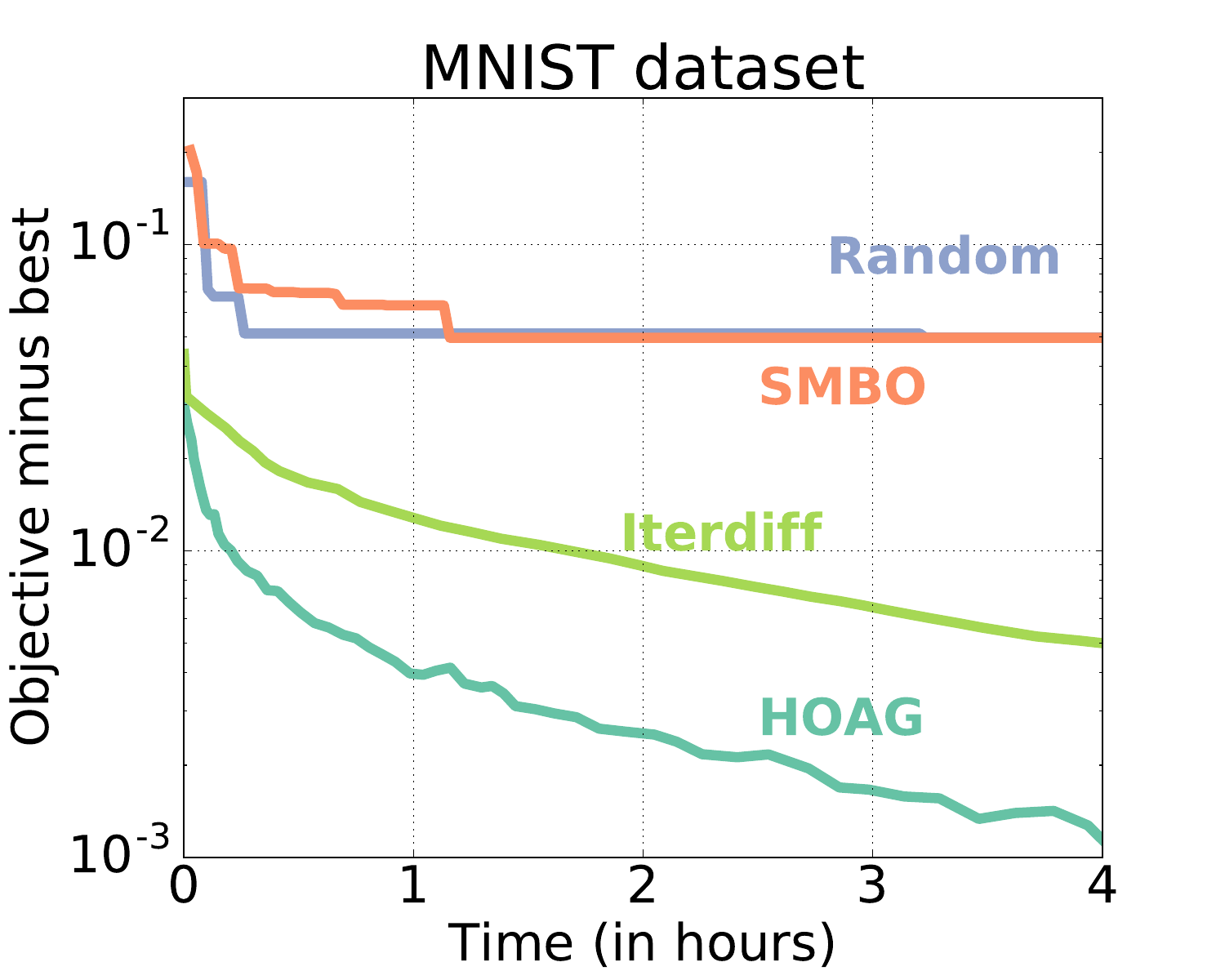}
\center \includegraphics[width=.7\linewidth]{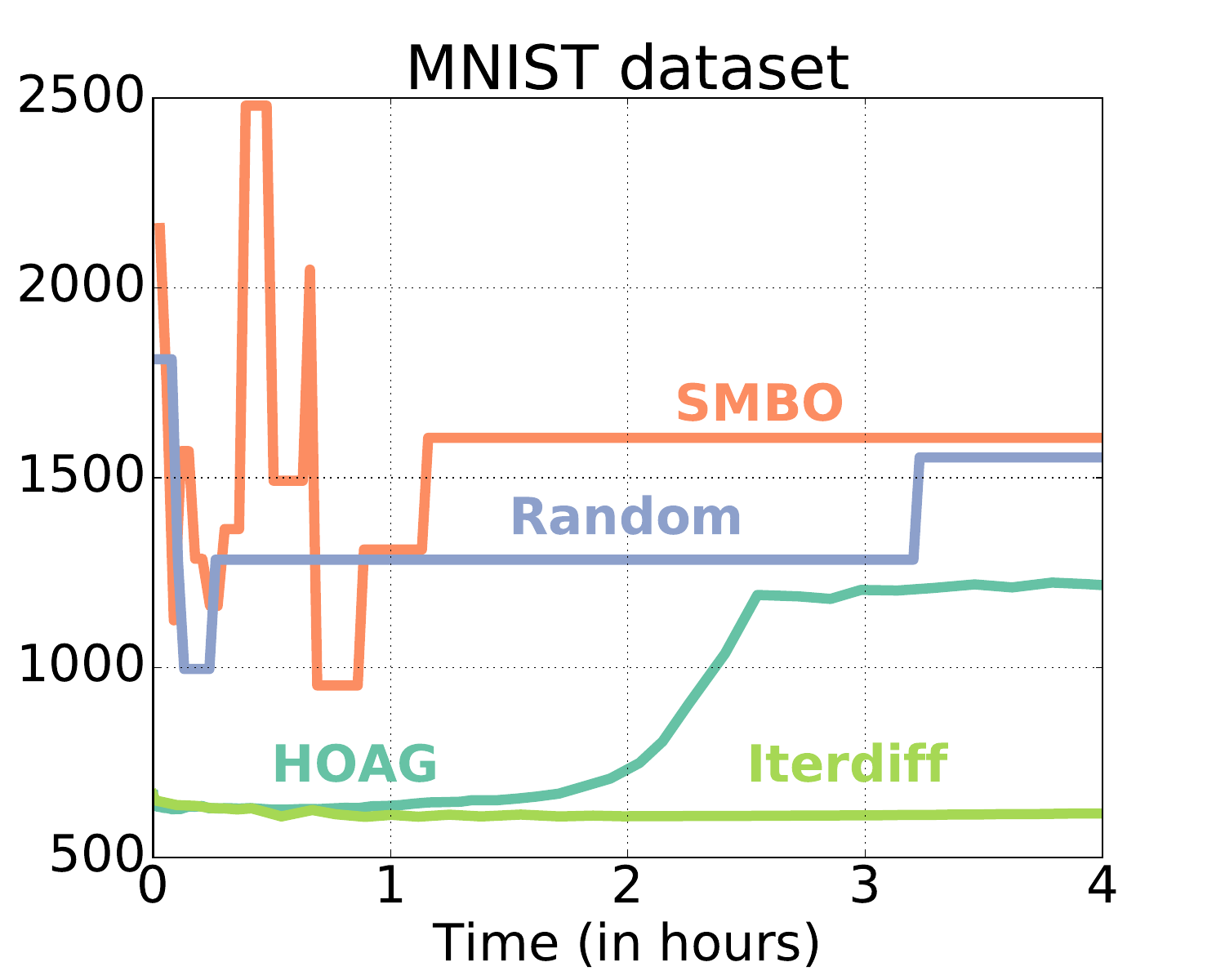}
\caption{{\bf Hyperparameter optimization methods}. Top row: suboptimality of the different methods in terms of the test loss. Bottom row: loss measured on a validation set for the different methods.}\label{fig:bench_ho_appendix}
\end{figure}
We compared the different methods considered on the MNIST dataset (except Grid search due to the very high dimensionality of the hyperparameter space) in terms of test and validation loss versus time. Results can be seen in Fig~\ref{fig:bench_ho_appendix}. In the top column we show the test loss as a function of time. As expected, only gradient-based hyperparameter optimization methods provide satisfactory results on this very high-dimensional problem. We can also see how \HOAG\ yields much better results than any of the other methods.

In the bottom row we display instead the validation loss as a function of time. In this plot we can see how the validation loss increases as the test loss decreases, providing a clear sign of overfit in the design of the model. In order to reduce this problem, we could add constraints or reduce the number of hyperparameters. However, this would change the experimental design while the goal of this experiment was to compare the methods in a similar setting to that of~\citep{MacDuvAda2015hyper}.

\end{document}